\documentclass[11pt]{article}

\usepackage[margin=1in]{geometry}

\usepackage{times}
\usepackage[round]{natbib}
\usepackage{amsmath,amssymb,amsfonts,amsthm,mathtools}
\usepackage{booktabs,tabularx}
\usepackage{graphicx}
\usepackage{float}
\usepackage{microtype}
\usepackage{url}
\usepackage{hyperref}
\usepackage{xcolor}
\usepackage{enumitem}

\newtheorem{theorem}{Theorem}
\newtheorem{proposition}[theorem]{Proposition}
\newtheorem{corollary}[theorem]{Corollary}
\newtheorem{lemma}[theorem]{Lemma}

\newcommand{\E}{\mathbb E}
\newcommand{\M}{\mathcal M}

\newcommand{\Ecal}{\mathcal E}
\newcommand{\Fcal}{\mathcal F}
\newcommand{\sgn}{\operatorname{sgn}}

\title{\bf Unbiased Gradients, Moving Stability Boundaries:\\ Exact Mini-Batch Geometry in Linear Self-Attention}

\author{Krishnakumar Balasubramanian\\ Department of Statistics, University of California, Davis\\ \texttt{kbala@ucdavis.edu}}
\date{}

\begin{document}

\maketitle

\begin{abstract}
Unbiased stochastic gradients can match the full gradient in expectation
while changing finite-step stability.  We study this effect in one-layer
linear self-attention for in-context linear regression, where shared-mode
mini-batch training reduces exactly to a random two-factor map.  The
full-batch map preserves an elliptic region, whereas sampled curvature and
target correlation create batch-dependent stability boundaries.  Their
combined fluctuation is centered at the parameter-update level but induces a
strictly outward drift of the full-batch boundary coordinate.  We derive an
exact one-step crossing criterion, adaptive concentration bounds, and
conditional drift identities, and prove almost-sure escape in a balanced
switching model.  For independent long prompts, mode leakage decays at the
inverse square root of prompt length.  For finite-token softmax attention, we
derive an explicit boundary-crossing law governed by an effective context
length that shrinks as attention scores concentrate.  Experiments in ambient
attention objectives verify the reduced predictions, and four trained causal
softmax Transformers show larger mean curvature shifts for smaller
first-step batches, consistent with the predicted nonlinear outward bias.
\end{abstract}

\section{Introduction}
\label{sec:intro}
Reliable optimization is essential for transformer scaling. At large learning rates, training runs can exhibit loss spikes, rapid growth in attention or output logits, oscillations, and divergence. These failures restrict usable optimization settings and waste substantial computation. Practical recipes use warmup, weight decay, normalization, and careful parameterization to mitigate these instabilities \citep{gilmer2022a,wortsman2024small,zhai2023stabilizing}. Edge-of-stability studies further show that gradient descent with finite step sizes often operates near or beyond its classical local stability threshold \citep{cohen2021gradient,arora2022understanding,damian2023self,regis2026rod}. These observations motivate understanding how individual finite updates can move an apparently stable run to an unstable one.

Specifically, we focus on a concrete question: \emph{can mini-batch sampling drive an iterate outside a region preserved by the corresponding full-batch dynamics, and can the crossing probability be quantified?} An unbiased mini-batch gradient matches the full gradient in expectation, but this agreement does not ensure that individual updates preserve the full-batch invariant region. Most analyses of transformers learning in-context predictors use gradient flow, population risk, or expected gradients \citep{vonoswald2023gradient,ahn2023preconditioned,zhang2024trained,wu2024how,zhang2024context}. Gradient flow omits finite-step effects, while population and expected-gradient analyses suppress batch-to-batch fluctuations. Addressing our question therefore requires studying finite stochastic updates directly.

We do so in a deliberately minimal setting: mini-batch gradient descent in a one-layer linear self-attention model of in-context linear regression. This model retains a bilinear attention parameterization, prompt-dependent empirical geometry, and nonlinear large-step dynamics. Under shared-mode mini-batch sampling,\footnote{A mode is a distinguished spectral pattern of the model whose amplitude and training dynamics can be studied separately.} the updates preserve a signed-spectral plane and reduce there to a random two-factor map. This reduction permits an exact analysis of finite-step training without a small-step approximation or local linearization.

The corresponding full-batch dynamics are governed by the scalar factorization map studied, for example, in \citet{chen2023stability,liang2026gradient,moon2026state}. We leverage its known sharp basin with a nested-ellipse geometry as the deterministic baseline, and introduce a mechanism which we call as  \emph{moving stability boundaries}. In the compact regime, the full-batch map preserves the interior of an invariant ellipse. Mini-batch sampling introduces variation in both curvature and target correlation. A sampled map that remains in the compact regime admits its own invariant ellipse: sampled curvature changes its scale, and sampled target correlation changes its shape. Other batches may fall outside the compact regime altogether. Successive updates thus switch among batch-dependent geometries and need not preserve the full-batch interior.

We show that, relative to the full-batch map, curvature and target fluctuations enter through a single centered, state-dependent scalar. This representation allows us quantify how crossing probabilities depend jointly on batch size, step size, and the current state. Our comparisons across batch sizes hold the step size and prompt population fixed to isolate the effect of sampling. The reduction also reveals a systematic effect of batch variation: at any point on the deterministic boundary where this scalar has nonzero variance, the full-batch boundary coordinate has a strictly outward drift, even though the parameter update is unbiased.

This mechanism is related to batch-dependent sharpness and the stochastic edge of stability \citep{wu2018sgd,lee2023batchsharpness,agarwala2024stochastic}, as well as multiplicative-noise theory, which distinguishes mean-square growth from growth along typical trajectories \citep{hodgkinson2021multiplicative,ziyin2022strength}. Our focus on exact, state-dependent boundary crossings also complements studies of catapult phases, balancing, self-stabilization, bifurcations, sustained SGD oscillations, and chaos \citep{lewkowycz2020large,wang2022large,zhu2023understanding,song2023trajectory,andriushchenko2023sgd,lu2024benign,kong2020stochasticity,herrmann2022chaotic}. In summary, in this work we make the following \textbf{contributions}:
\begin{itemize}[leftmargin=*, noitemsep]
    \item Section~\ref{sec:reduction} derives the exact shared-mode
    mini-batch reduction (Propositions~\ref{prop:transformer-reduction}
    and~\ref{prop:shared-mode}).  Lemma~\ref{thm:moving-geometry} shows that
    each compact-regime batch has its own rescaled stability ellipse and that
    all batch variation enters through one centered, state-dependent scalar.
    Proposition~\ref{prop:approx-shared-mode} and
    Figure~\ref{fig:approx-lsa} extend the picture to independent prompts,
    with leakage decaying as the inverse square root of prompt length.
    \item Section~\ref{sec:crossing} gives an exact statewise safe interval
    (Proposition~\ref{thm:safe-interval}) and exponential batch-size crossing
    bounds (Corollary~\ref{cor:batch-bound} and
    Theorem~\ref{thm:sequential-crossing}), validated in
    Figure~\ref{fig:crossing}.
    \item Section~\ref{sec:moments} proves outward boundary drift despite
    unbiased parameter updates (Proposition~\ref{thm:moments}), separates
    frozen-state from pathwise growth
    (Proposition~\ref{prop:frozen-vs-moment} and
    Figure~\ref{fig:transverse-dichotomy}), and establishes an almost-sure
    escape criterion and two-type result
    (Theorem~\ref{thm:as-escape} and
    Corollary~\ref{cor:two-type-as-escape}).
    \item Section~\ref{sec:softmax} extends the mechanism to finite-token
    softmax attention.  Proposition~\ref{prop:softmax-clt} gives the explicit
    central-limit and boundary-crossing law and identifies the effective
    context length.  The projected and ambient predictions are tested in
    Figures~\ref{fig:softmax-finite-token}--%
    \ref{fig:softmax-ambient-interior-shadowing}.
    \item Section~\ref{sec:experiments} and
    Appendix~\ref{app:trained-transformer-kick} connect the exact theory to
    ambient objectives.  Figure~\ref{fig:approx-lsa} verifies the
    independent-prompt approximation, while
    Figure~\ref{fig:transformer-theory-ablations} shows the finite-reference
    gradient scale and replicated batch-size trends in fixed-batch curvature
    and matched continuation across four trained causal softmax Transformers (with 1.4M parameters) on WikiText-2 data.
\end{itemize}

\section{Transformer mini-batch reduction}
\label{sec:reduction}

Consider a single noiseless linear-regression prompt
\[
E=
\begin{pmatrix}
x_1&\cdots&x_N&x_q\\
y_1&\cdots&y_N&0
\end{pmatrix},
\qquad y_i=\langle w_\star,x_i\rangle,~\text{for}~i=1,\ldots N, \quad
y_q=\langle w_\star,x_q\rangle.
\]
For the active blocks of a one-layer linear self-attention model
\citep{zhang2024trained}, the query prediction can be written
\[
    \widehat y(u)=u^\top H u,
    \qquad
    H=X_q\otimes G,\qquad
    G=\frac{EE^\top}{2N},
\]
where \(u\) vectorizes the active value and key--query parameters and \(X_q\)
has nonzero eigenvalues \(\pm\|x_q\|\).  The loss is
\(\mathcal L(u)=\frac12(u^\top Hu-y_q)^2\).

\begin{proposition}[Single-prompt spectral plane]
\label{prop:transformer-reduction}
Let \(Gp=\gamma p\) with \(\gamma>0\), and set
\(\lambda=\|x_q\|\gamma\).  The two eigenvectors of \(H\) with eigenvalues
\(\pm\lambda\) span a plane that is exactly invariant under every finite
gradient-descent step.  In suitable coordinates on this plane and after a
fixed rescaling, gradient descent is
\begin{equation}
\label{eq:phi}
    \Phi_\mu(a,b)
    =
    \bigl(a-(ab-\mu)b,\ b-(ab-\mu)a\bigr),
    \qquad
    \mu=2\eta\lambda|y_q|.
\end{equation}
\end{proposition}

The parameter \(\mu\) combines the learning rate, prompt curvature, query
norm, and target magnitude.  Appendix~\ref{app:transformer-proof} derives
this exact reduction.  The next subsection extends it to mini-batch updates.

\subsection{Shared modes for prompt batches}

\begin{proposition}[Exact shared-mode transformer mini-batches]
\label{prop:shared-mode}
For prompt \(i\), let
\(\mathcal L_i(u)=\frac12(u^\top H_i u-y_i)^2\).  Suppose the symmetric
matrices \(H_i\) share orthonormal vectors \(h_+,h_-\) satisfying
\[
    H_i h_+=\gamma_i h_+,\qquad
    H_i h_-=-\gamma_i h_-,
    \qquad \gamma_i\ne0.
\]
Then \(\operatorname{span}\{h_+,h_-\}\) is invariant under every prompt
gradient and hence under every mini-batch update.  If
\(a=c_++c_-\), \(b=c_+-c_-\), one batch step with transformer learning rate
\(\eta_{\rm tr}\) is
\[
\begin{split}
    a^+=a-2\eta_{\rm tr}
    \left[\frac1m\sum_{i\in\mathsf B}
    \gamma_i(\gamma_iab-y_i)\right]b,\quad
    b^+=b-2\eta_{\rm tr}
    \left[\frac1m\sum_{i\in\mathsf B}
    \gamma_i(\gamma_iab-y_i)\right]a.
\end{split}
\]
Equivalently, with \(\eta=2\eta_{\rm tr}\), this is gradient descent on
\begin{equation}
\label{eq:sample-loss}
    \ell_i(a,b)=\frac12(\gamma_iab-y_i)^2.
\end{equation}
\end{proposition}

For \(H_i=X_{q,i}\otimes G_i\), the shared-mode conditions hold whenever
\(x_{q,i}=t_i v\) for a fixed unit direction \(v\) (the signs \(t_i\) may
vary) and the \(G_i\) share an eigenvector.  The signed \(\gamma_i\) encodes
the query orientation.  The next result quantifies deviations from this
structure for independent long prompts.

\begin{proposition}[Approximate shared mode for independent prompts]
\label{prop:approx-shared-mode}
Let \(x_{q,i}=t_i v\), define $r_\pm=2^{-1/2}(v,\pm1)$ and $h_\pm=r_\pm\otimes p$. Let \(P\) project onto
\(\mathcal S=\operatorname{span}\{h_+,h_-\}\).  Suppose
\(\bar Gp=g p\), \(\|p\|=1\), and write $g_i=p^\top G_i p$, $\delta_i=\|(I-pp^\top)G_i p\|$ and $\gamma_i=t_i g_i$. For \(u=c_+h_++c_-h_-\), with \(a=c_++c_-\) and \(b=c_+-c_-\),
the projected prompt gradient is exactly the gradient of
\(\frac12(\gamma_iab-y_i)^2\), while
\begin{equation}
\label{eq:gradient-leakage}
\|(I-P)\nabla\mathcal L_i(u)\|
=2|\gamma_iab-y_i|\,|t_i|\,\delta_i\|u\|.
\end{equation}
If the training examples are Gaussian regression samples
\(z=(x,\langle w_\star,x\rangle)\) with covariance \(C\), then for
\(\bar G=C/2\),
\begin{equation}
\label{eq:gram-concentration}
\E\|G_i-\bar G\|_F^2
=\frac{(\operatorname{tr}C)^2+\|C\|_F^2}{4N}
+ \frac{\|q_i\|^4}{4N^2},
\qquad q_i=(x_{q,i},0),
\end{equation}
and therefore \(\E\delta_i^2=O(N^{-1})\).
\end{proposition}

These prompts have a shared population mode and \(O(N^{-1/2})\)
root-mean-square gradient leakage on bounded states.  Finite-horizon
shadowing (Proposition~\ref{prop:shadowing}) also depends on the
complementary propagator.  For residual
\(\varepsilon_i=u^\top H_i u-y_i\) and a complementary eigenvector
\(H_i h'=\lambda'h'\), $DT_i(u)h'=(1-2\eta_{\rm tr}\varepsilon_i\lambda')h'.$ The signed pair \(\pm\lambda'\) has multipliers
\(1\mp2\eta_{\rm tr}\varepsilon_i\lambda'\), one of which exceeds one unless
\(\varepsilon_i\lambda'=0\).  The complementary dynamics are therefore not
pointwise contractive in general.  Figure~\ref{fig:approx-lsa} reports their
finite-horizon normal amplification.

\subsection{Normalization and the full-batch boundary}

For the mode-wise losses in \eqref{eq:sample-loss}, let $\bar c=\frac1n\sum_i\gamma_i^2$, $\bar h=\frac1n\sum_i\gamma_i y_i$, $\mu=\eta\bar h$, and for a batch \(\mathsf B\),
\[
    \alpha_{\mathsf B}
    =
    \frac{m^{-1}\sum_{i\in\mathsf B}\gamma_i^2}{\bar c},
    \qquad
    \nu_{\mathsf B}
    =
    \eta\,m^{-1}\sum_{i\in\mathsf B}\gamma_i y_i.
\]
Under the fixed normalization
\((A,B)=\sqrt{\eta\bar c}\,(a,b)\), one batch step is
\begin{equation}
\label{eq:Psi}
    \Psi_{\alpha,\nu}(A,B)
    =
    \bigl(A-(\alpha AB-\nu)B,\,
          B-(\alpha AB-\nu)A\bigr).
\end{equation}

The full-batch map is \(\Psi_{1,\mu}=\Phi_\mu\).  Write $e=AB-\mu,$ $w=A-B$ and $D_\mu(A,B)=w^2-(2-\mu)(2-e).$ The deterministic identity
\begin{equation}
\label{eq:D-deterministic}
    D_\mu(\Phi_\mu(A,B))
    =
    q_\mu(e)D_\mu(A,B),
    \qquad q_\mu(e)=e^2+\mu e+1,
\end{equation}
is verified in Appendix~\ref{app:deterministic}.  The global basin
interpretation and nested-state geometry are established by
~\citet{chen2023stability}, \citet{liang2026gradient} and \citet{moon2026state}.  For \(0<\mu<2\),
\(D_\mu=0\) is the invariant ellipse
\begin{equation}
\label{eq:ellipse}
    \Ecal_\mu
    =
    \left\{(A,B):
    \frac{(A+B)^2}{2+\mu}
    +
    \frac{(A-B)^2}{2-\mu}=4
    \right\}.
\end{equation}
The region \(D_\mu<0\) is the full-batch bounded side.  Under the fixed
normalization, \(\eta\) determines \(\mu\), \(\nu_{\mathsf B}\), and
\((A,B)\); at fixed data and original coordinates \((a,b)\), it therefore
changes both the boundary and the normalized state.  With \(\eta\) fixed,
increasing \(m\) concentrates the sampled coefficients.  At a fixed
original state, the RMS update deviation scales as \(\eta/\sqrt m\).
Stability also depends on the \(\eta\)-dependent deterministic update.

\begin{lemma}[Moving batch geometry and effective noise]
\label{thm:moving-geometry}
For \(\alpha>0\), the scaling
\(T_\alpha(A,B)=\sqrt\alpha(A,B)\) conjugates
\(\Psi_{\alpha,\nu}\) to \(\Phi_\nu\).  Hence every batch with
\(|\nu|<2\) has its own invariant ellipse
\[
    \frac{\alpha(A+B)^2}{2+\nu}
    +
    \frac{\alpha(A-B)^2}{2-\nu}=4.
\]
Relative to the full-batch map, define
\begin{equation}
\label{eq:rho}
    \rho_{\mathsf B}(A,B)
    =
    \nu_{\mathsf B}-\mu
    -
    (\alpha_{\mathsf B}-1)AB.
\end{equation}
Then the exact batch update is
\begin{equation}
\label{eq:Psi-rho}
    \Psi_{\alpha_{\mathsf B},\nu_{\mathsf B}}(A,B)
    =
    \Phi_{\mu+\rho_{\mathsf B}(A,B)}(A,B)
    =
    \Phi_\mu(A,B)+\rho_{\mathsf B}(A,B)(B,A).
\end{equation}
For an unbiased batch draw,
\(\E[\rho_{\mathsf B}(A,B)\mid A,B]=0\).
\end{lemma}

\section{Mini-batch boundary crossing}
\label{sec:crossing}

\paragraph{Moving stability boundaries.}
Lemma~\ref{thm:moving-geometry} shows that each compact-regime batch has an
ellipse determined by \((\alpha_{\mathsf B},\nu_{\mathsf B})\).  Sampling a
new batch changes these coefficients, so the ellipse is rescaled and
reshaped from one update to the next.  Relative to the fixed full-batch
ellipse, this variation is encoded by \(\rho_{\mathsf B}(A,B)\); its
dependence on \(AB\) also makes the crossing margin state dependent.
Proposition~\ref{thm:safe-interval} gives the corresponding exact crossing
interval, while Corollary~\ref{cor:batch-bound} and
Theorem~\ref{thm:sequential-crossing} give one-step and finite-horizon
crossing bounds.

Fix \(0<\mu<2\), let \(e=AB-\mu\), \(D=D_\mu(A,B)\), and set $H_\mu(A,B)=A^2+B^2-\mu AB=D+4-\mu^2$ and $J_\mu(A,B)=(2e+\mu)D+(4-\mu^2)e.$
\begin{proposition}[Exact state-dependent safe interval]
\label{thm:safe-interval}
For a batch with effective perturbation \(\rho\),
\begin{equation}
\label{eq:D-rho}
    D_\mu(A^+,B^+)
    =
    H_\mu(A,B)\rho^2
    -
    J_\mu(A,B)\rho
    +
    q_\mu(e)D.
\end{equation}
If \(D<0\) and \((A,B)\ne(0,0)\), define
\begin{equation}
\label{eq:rho-pm}
    \rho_\pm
    =
    \frac{J_\mu
    \pm\sqrt{J_\mu^2-4H_\mu q_\mu(e)D}}
    {2H_\mu}.
\end{equation}
Then \(\rho_-<0<\rho_+\), and the batch remains on the full-batch bounded
side if and only if \(\rho\in(\rho_-,\rho_+)\).  It crosses to
\(D_\mu>0\) if and only if \(\rho<\rho_-\) or \(\rho>\rho_+\).
\end{proposition}

The interval \((\rho_-,\rho_+)\) is the exact robustness interval.  On
\(D=0\), \eqref{eq:D-rho} becomes
\(D_\mu(A^+,B^+)=(4-\mu^2)\rho(\rho-e)\); the stated outer ranges of
\(\rho\) give \(D_\mu(A^+,B^+)>0\).

\begin{corollary}[Batch-size crossing bound]
\label{cor:batch-bound}
At a nonzero interior state, for sampling with replacement, let
\[
Z_i(p)
=
\eta(\gamma_i y_i-\bar h)
-
p\left(\frac{\gamma_i^2}{\bar c}-1\right),
\qquad p=AB.
\]
Then \(\rho_{\mathsf B}=m^{-1}\sum_{i\in\mathsf B}Z_i(p)\).  Suppose the
single-example fluctuation \(Z_i(p)\), under uniform sampling, has
sub-Gaussian proxy variance \(\kappa(p)^2\).  Then
\begin{equation}
\label{eq:crossing-bound}
\Pr(D_\mu(A^+,B^+)\ge0\mid A,B)
\le
2\exp\left(
-\frac{m\,\tau_\mu(A,B)^2}{2\kappa(p)^2}
\right),
\end{equation}
where \(\tau_\mu=\min\{\rho_+,-\rho_-\}\).
\end{corollary}

For a finite prompt collection, Hoeffding's lemma gives an explicit proxy
variance.
If
\[
\kappa_0=\frac{\eta}{2}
\left(\max_i\gamma_i y_i-\min_i\gamma_i y_i\right),\qquad
\kappa_1=\frac12
\left(\max_i\frac{\gamma_i^2}{\bar c}
-\min_i\frac{\gamma_i^2}{\bar c}\right),
\]
then one may take $\kappa(p)\le \kappa_0+|p|\kappa_1.$ With unequal batch
curvatures, the concentration rate varies along the trajectory.  The
dependence on \(m\) in \eqref{eq:crossing-bound} applies at fixed step size,
state, and prompt population.  Its coefficient
\(\tau_\mu(A,B)^2/(2\kappa(AB)^2)\) depends on \(\eta\) through the
boundary, normalized state, and single-example fluctuation.

\begin{theorem}[Adaptive finite-horizon crossing bound]
\label{thm:sequential-crossing}
Let \(T=\inf\{k\ge0:D_\mu(A_k,B_k)\ge0\}\), assume \(D_0<0\), and draw a
fresh batch of size \(m\) with replacement at every step.  For \(K\ge1\) and
\(\chi>0\), define
\[
\mathcal G_{K,\chi}
=
\left\{
\frac{\tau_\mu(A_k,B_k)}
{\kappa(A_kB_k)}
\ge\chi
\quad\text{for every }0\le k<T\wedge K
\right\}.
\]
At the absorbing origin set \(\tau_\mu=+\infty\), and use the convention
\(\tau_\mu/\kappa=+\infty\) whenever \(\kappa=0\).
Then
\begin{equation}
\label{eq:sequential-crossing}
\Pr(T\le K,\mathcal G_{K,\chi})
\le 2K\exp\left(-\frac{m\chi^2}{2}\right).
\end{equation}
\end{theorem}

The theorem bounds crossing on paths that retain normalized margin
\(\chi\).  In particular,
$\Pr(T\le K)\le
2Ke^{-m\chi^2/2}+\Pr(\mathcal G_{K,\chi}^{c})$.  A conditional bound given
\(\mathcal G_{K,\chi}\) divides the first term by
\(\Pr(\mathcal G_{K,\chi})\).  Margin erosion and the unconditional escape
time are not bounded here.
\begin{figure}[t]
\centering
\includegraphics[width=\linewidth]{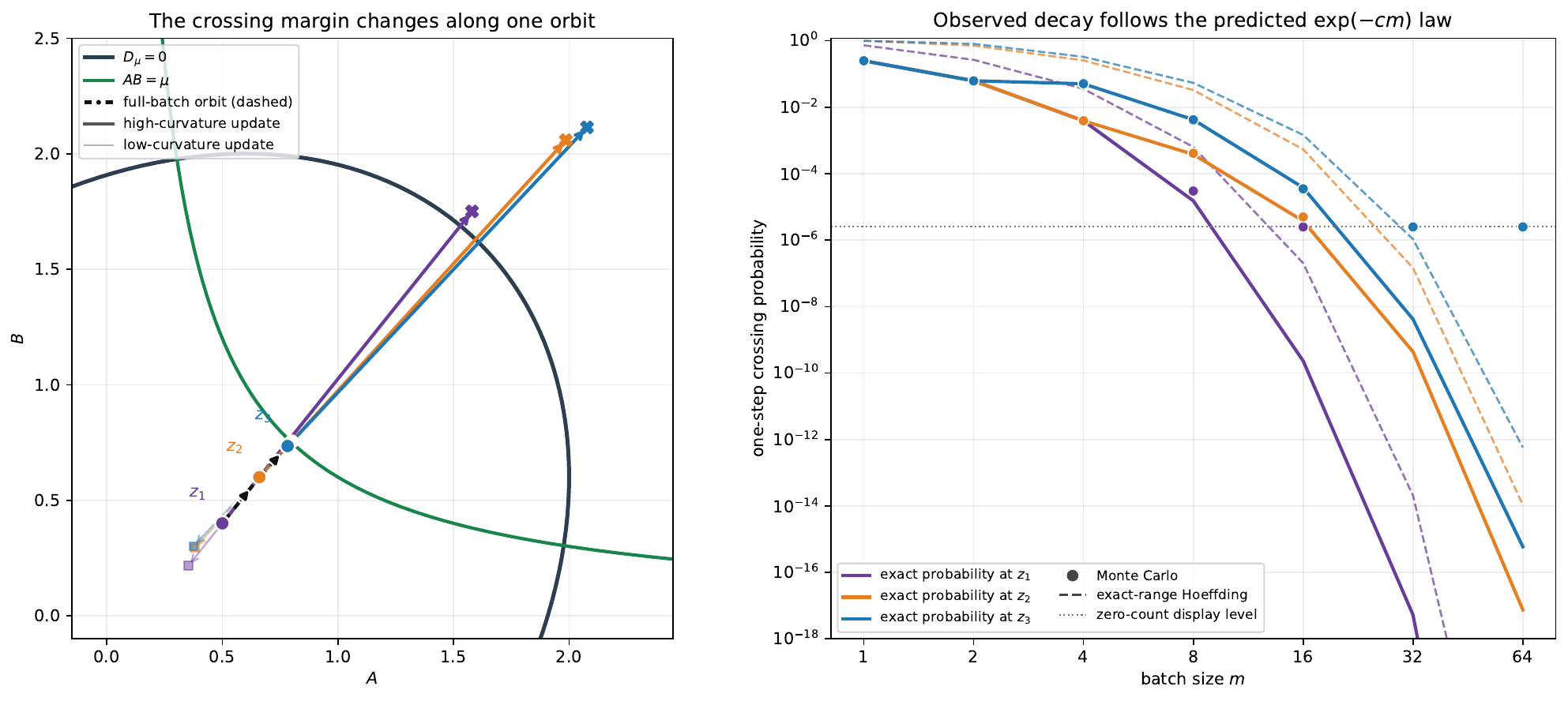}
\caption{State-dependent batch-size crossing probabilities with unequal
curvatures.  Left: three consecutive full-batch states, joined by the
outlined black dashed orbit, and their one-sample updates under two prompt types with
\((\alpha,\nu)\in\{(0.5,-4/15),(2.5,3.2)\}\); the high-curvature update
crosses \(D_\mu=0\).  Right: exact binomial crossing probabilities, Monte
Carlo estimates from \(2\times10^5\) batches, and the exact-range Hoeffding
bound for the two-point fluctuation at each state.  All three states exhibit
exponential decay in \(m\).  Appendix~\ref{app:experiments}
gives an exact noiseless, task-diverse attention-prompt realization of these
signed scalar types; the fixed-regressor construction in
Figure~\ref{fig:approx-lsa} has same-sign target correlations.}
\label{fig:crossing}
\end{figure}

\paragraph{Boundary crossing at a full-batch solution.}
At \(\theta_\mu=(\sqrt\mu,\sqrt\mu)\), with \(0<\mu<1\), a batch crosses
the ellipse exactly when
\begin{equation}
\label{eq:solution-crossing}
    \mu(\rho_{\mathsf B}(\theta_\mu)^2
    +2\rho_{\mathsf B}(\theta_\mu))>2.
\end{equation}
Equation~\eqref{eq:solution-crossing} is the one-step stochastic crossing
condition at this locally stable full-batch solution.

\section{Stability under mini-batch sampling}
\label{sec:moments}

Let \(\Fcal_k\) denote the trajectory history, define
\(\rho_k=\rho_{\mathsf B_k}(A_k,B_k)\) and
\(v_k=\E[\rho_k^2\mid\Fcal_k]\), and let \(e_k=A_kB_k-\mu\),
\(w_k=A_k-B_k\), and \(D_k=D_\mu(A_k,B_k)\).

\begin{proposition}[Exact conditional drift identities]
\label{thm:moments}
For conditionally unbiased mini-batches,
\begin{align}
\E[(A_{k+1},B_{k+1})\mid\Fcal_k]
&=\Phi_\mu(A_k,B_k), \label{eq:mean-map}\\
\operatorname{Cov}((A_{k+1},B_{k+1})\mid\Fcal_k)
&=
v_k
\begin{pmatrix}
B_k^2&A_kB_k\\
A_kB_k&A_k^2
\end{pmatrix}, \label{eq:cov-map}\\
\E[w_{k+1}^2\mid\Fcal_k]
&=\bigl((1+e_k)^2+v_k\bigr)w_k^2, \label{eq:w2}\\
\E[D_{k+1}\mid\Fcal_k]
&=\bigl(q_\mu(e_k)+v_k\bigr)D_k
 +(4-\mu^2)v_k. \label{eq:D-drift}
\end{align}
\end{proposition}

Proposition~\ref{thm:moments} gives the conditional-moment consequence of
moving stability boundaries.  On the deterministic boundary,
\(\E[D_{k+1}\mid\Fcal_k]=(4-\mu^2)v_k>0\) for nonzero batch variance:
the conditional mean parameter update follows the full-batch map, while the
basin coordinate has outward conditional drift.

\begin{proposition}[Frozen-state log contraction, moment growth, and longitudinal forcing]
\label{prop:frozen-vs-moment}
Let \(\theta_\mu=(\sqrt\mu,\sqrt\mu)\), \(0<\mu<1\), and write
\(\rho=\rho_{\mathsf B}(\theta_\mu)=\nu_{\mathsf B}-\alpha_{\mathsf B}\mu\), 
with \(\E\rho=0\) and \(v=\E\rho^2\).  One batch sends $\theta_\mu^+=(1+\rho)\theta_\mu,$ and $\delta w^+=(1-\rho)\delta w$ to first order in transverse imbalance.  Consequently,
\begin{equation}
\label{eq:local-dichotomy}
\E\|\theta_\mu^+-\theta_\mu\|^2=2\mu v,\qquad
\E[(\delta w^+)^2]=(1+v)(\delta w)^2.
\end{equation} 
If \(\rho<1\) almost surely and is nonconstant, then $\Lambda_{\rm fr}:=\E\log(1-\rho)<0,$ and $\Lambda_{\rm rms}:=\frac12\log(1+v)>0.$ If additionally \(|\rho|\le r<1\), then
\[
\left|\Lambda_{\rm fr}+\frac v2\right|
\le \frac{\E|\rho|^3}{3(1-r)}.
\]
\end{proposition}

These are one-step statistics at the fixed state \(\theta_\mu\).
Trajectory-level transverse stability is governed by the state-dependent
cocycle below.  At \(\theta_\mu\), the deterministic multipliers are one
transversely and \(1-2\mu\) along the balanced direction.

The exact pathwise transverse cocycle is
\begin{equation}
\label{eq:random-cocycle}
    w_{k+1}
    =
    (1+\alpha_kA_kB_k-\nu_k)w_k.
\end{equation}
Along \(A_k=B_k=r_k\),
\[
r_{k+1}=r_k(1+\nu_k-\alpha_kr_k^2),
\qquad
|\delta w_n|
=|\delta w_0|\prod_{k<n}|1+\alpha_kr_k^2-\nu_k|.
\]
Along a balanced trajectory,
\(\Lambda_n=n^{-1}\sum_{k<n}\log|1+\alpha_kr_k^2-\nu_k|\) is the finite-time
transverse Lyapunov rate; \(\Lambda_{\rm fr}\) is the frozen-state log rate
at \(\theta_\mu\).

\begin{figure}[t]
\centering
\includegraphics[width=\linewidth]{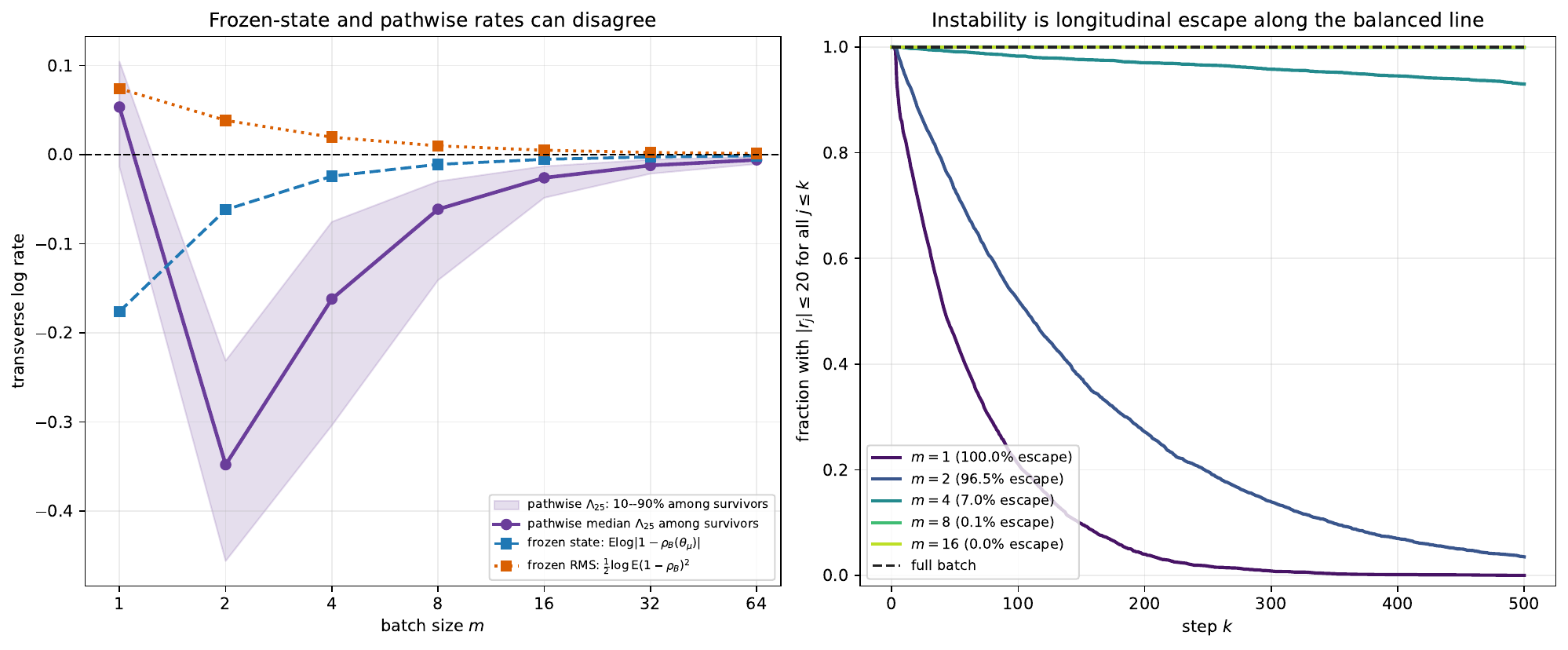}
\caption{Frozen-state and pathwise transverse-growth statistics for the
centered two-type population
\((\alpha,\nu)\in\{(0.5,0.1),(3,2.6)\}\), with probabilities \(0.8,0.2\)
and \(\mu=0.6\).  Left: exact frozen one-step log and RMS rates,
together with the median and 10--90\% range of the actual
\(\Lambda_{25}\) among \(10^5\) trajectories that remain inside the ellipse.
At \(m=1\), the frozen log rate is \(-0.176\); the survivor-conditioned
pathwise median is \(0.053\), and about \(85\%\) are positive.  Right:
balanced trajectories escape in nearly \(100\%\), \(96\%\), and \(7\%\) of
the \(m=1,2,4\) runs by step 500 (4000 realizations).
Corollary~\ref{cor:two-type-as-escape} shows that every plotted finite batch
size escapes almost surely on an infinite horizon.}
\label{fig:transverse-dichotomy}
\end{figure}

For these simulations, Theorem~\ref{thm:sequential-crossing} gives
finite-horizon bounds.  With \(K=500\), the exact two-point range, and
\(\chi=1.75\), no \(m=8\) or \(m=16\) trajectory in
\(\mathcal G_{K,\chi}\) crosses; the corresponding bounds are
\(4.79\times10^{-3}\) and \(2.29\times10^{-8}\).  Most \(m=8\) paths and
essentially all \(m=16\) paths retain this margin.  The smaller-batch
crossings are preceded by margin erosion.  Appendix~\ref{app:experiments}
reports the realized margins.

\begin{theorem}[Compact-invariant-set dichotomy]
\label{thm:as-escape}
Let \(\xi_k\) be i.i.d.\ with full support on a finite set
\(\mathcal J\), and consider the balanced random dynamics
\[
    x_{k+1}=g_{\xi_k}(x_k),
    \qquad
    g_j(x)=x(1+\nu_j-\alpha_jx)^2,
    \qquad \alpha_j,\nu_j>0.
\]
Let \(\mathcal L\) be the countable set of \(x>0\) that some finite
composition maps to zero.  If \(\{0\}\) is the only nonempty compact
\(K\subset[0,\infty)\) satisfying \(g_j(K)\subseteq K\) for every
\(j\in\mathcal J\), then, for every \(x_0>0\) with
\(x_0\notin\mathcal L\), $ x_k\to\infty$, almost surely. Conversely, every such nontrivial \(K\) traps all switching trajectories
started in \(K\).
\end{theorem}

\begin{corollary}[Two-type escape at every finite batch size]
\label{cor:two-type-as-escape}
For the population in Figure~\ref{fig:transverse-dichotomy}, a batch with
\(j\) high-type prompts has
\[
\alpha_j=\frac12+\frac{5j}{2m},
\qquad
\nu_j=\frac1{10}+\frac{5j}{2m}.
\]
For every finite \(m\), the only nonempty compact common forward-invariant
set is \(\{0\}\).  Hence every positive nonlanding balanced initialization
escapes almost surely.  Moreover, \(x_0=A_0^2=B_0^2=3/5\) is nonlanding whenever
\(5\nmid m\), including all displayed batch sizes.
\end{corollary}

The conclusion for every finite batch size applies to the displayed
normalized type parameters.  Changing the step size for the same prompt
population changes the \(\nu_j\), so the corollary does not give a
corresponding statement for every step size.

Theorem~\ref{thm:sequential-crossing} and
Corollary~\ref{cor:two-type-as-escape} describe finite-horizon retention and
asymptotic balanced-line escape.  For \(K\ll e^{m\chi^2/2}\), the bound is
small on paths retaining margin \(\chi\); margin-eroding paths remain
uncontrolled.  At
\(m=16\), an all-high batch has probability
\(0.2^{16}\approx6.55\times10^{-12}\).  Quasi-stationary analysis may
characterize the conditional dynamics inside the ellipse.  Almost-sure
off-line behavior remains open.

\paragraph{Common fixed points and parameter convergence.}
General stochastic-gradient trajectories need not converge to a parameter
limit \citep{zhang2022neural}.  In the present model, the following
proposition gives a necessary condition for nonzero parameter convergence.

\begin{proposition}[Batch disagreement prevents nonzero convergence]
\label{thm:no-convergence}
Suppose mini-batches are drawn i.i.d.\ from a finite collection
\(\mathfrak B\), every batch has positive probability, and
\(\alpha_{\mathsf B}>0\).  If
\((A_k,B_k)\to(A_\infty,B_\infty)\ne(0,0)\), then
\begin{equation}
\label{eq:common-interpolation}
    \alpha_{\mathsf B}A_\infty B_\infty
    =
    \nu_{\mathsf B}
    \qquad\text{for every }\mathsf B\in\mathfrak B.
\end{equation}
Therefore nonzero parameter convergence is impossible unless all recurring
batches share the same product minimizer
\(\nu_{\mathsf B}/\alpha_{\mathsf B}\).
\end{proposition}

If every ratio equals \(p_\star\), every point of \(AB=p_\star\) is fixed by
every batch map, so the common-product condition also characterizes common
fixed points.  The proposition separates the interpolating regime, where
stochastic gradients vanish on a common solution, from prompt collections
with irreducible batch disagreement.

\section{Finite-token softmax in-context learning}
\label{sec:softmax}

Consider population softmax attention with prompt labels.  Let
\(\widetilde z=(x,y)\), where \(x\sim N(0,I_d)\) and
\(y=w_\star^\top x\). For a fixed query token
\(\widetilde z_q=(x_q,0)\), Gaussian tilting gives
\begin{equation}
\label{eq:gaussian-softmax}
\frac{\E[\exp(\widetilde z_q^\top Q\widetilde z')\widetilde z']}
     {\E[\exp(\widetilde z_q^\top Q\widetilde z')]}
=\Sigma_z Q^\top\widetilde z_q,
\qquad\text{where}\qquad
\Sigma_z=
\begin{pmatrix}
I_d&w_\star\\
w_\star^\top&\|w_\star\|^2
\end{pmatrix}.
\end{equation}

Let \(\ell=\|w_\star\|>0\), \(h=1+\ell^2\), \(c=\|x_q\|\), and
\[
r_\star=\frac{(w_\star,\ell^2)}{\ell h},\qquad
v=\frac{\widetilde z_q}{c}.
\]
Then \(\Sigma_zr_\star=hr_\star\),
\(\|r_\star\|^2=h^{-1}\), and $Z_j=r_\star^\top\widetilde z_j=\frac{y_j}{\ell}
\sim N(0,1).$ The direction \(r_\star\) therefore represents the normalized
prompt label.  On the
rank-one sector
\(\mathcal A=\sgn(y_q)\alpha r_\star^\top\),
\(Q=\beta vr_\star^\top\), the score is \(s=c\beta\), the population
prediction is \(\sgn(y_q)c\alpha\beta\), and the loss reduces to
\(\frac12(c\alpha\beta-|y_q|)^2\).  The resulting population gradient
descent map is \(\Phi_\mu\).

For finite \(N\), empirical token moments introduce an additional
fluctuation scale.  Define
\[
\pi_j(s)=\frac{e^{sZ_j}}{\sum_{\ell=1}^N e^{sZ_\ell}},\qquad
\widehat m_N(s)=\sum_j\pi_j(s)Z_j,\qquad
\widehat v_N(s)=\sum_j\pi_j(s)(Z_j-\widehat m_N(s))^2.
\]
If the ambient parameter step is \(\eta_{\rm amb}\), projection onto this
non-unit sector has effective scalar step
\(\eta_{\rm sm}=h\eta_{\rm amb}\).  With normalized factors
\(A=\sqrt{\eta_{\rm sm}}c\alpha\),
\(B=\sqrt{\eta_{\rm sm}}c\beta\), and
\(\mu=\eta_{\rm sm}c|y_q|\), the score is
\(s=B/\sqrt{\eta_{\rm sm}}\), and projected finite-token gradient descent is
the exact map
\begin{equation}
\label{eq:finite-softmax-map}
\widehat\Phi_{\mu,N}^{(\eta_{\rm sm})}(A,B)
=
\left(
A-\widehat e_N\sqrt{\eta_{\rm sm}}\,\widehat m_N(s),\
B-\widehat e_N A\widehat v_N(s)
\right),
\quad
\widehat e_N=\sqrt{\eta_{\rm sm}}A\widehat m_N(s)-\mu.
\end{equation}
As \((\widehat m_N,\widehat v_N)\to(s,1)\), the map converges to
\(\Phi_\mu\).  At finite \(N\), it also depends on the empirical moments
and \(\eta_{\rm sm}\).

\begin{proposition}[Softmax fluctuation and boundary layer]
\label{prop:softmax-clt}
For fixed \(s\),
\[
\sqrt N
\begin{pmatrix}\widehat m_N(s)-s\\ \widehat v_N(s)-1\end{pmatrix}
\ \Rightarrow\
\mathcal N\!\left(0,\Sigma_{\rm sm}(s)\right),\qquad
\Sigma_{\rm sm}(s)
=e^{s^2}
\begin{pmatrix}
1+s^2&s^3+2s\\
s^3+2s&s^4+4s^2+2
\end{pmatrix}.
\]
For \(\zeta=(A,B)\), let \(e=AB-\mu\) and
\[
J_\eta(\zeta)=
\begin{pmatrix}
-\sqrt{\eta_{\rm sm}}(2AB-\mu)&0\\
-\sqrt{\eta_{\rm sm}}A^2&-Ae
\end{pmatrix}.
\]
Then
\[
\sqrt N\bigl(\widehat\Phi_{\mu,N}^{(\eta_{\rm sm})}(\zeta)
-\Phi_\mu(\zeta)\bigr)
\Rightarrow
\mathcal N(0,J_\eta\Sigma_{\rm sm}J_\eta^\top).
\]
If \(\zeta_N\to\zeta_\star\),
\(\sqrt N D_\mu(\Phi_\mu(\zeta_N))\to-\delta\), and the limiting variance
\(\sigma_D^2\) of \(D_\mu\) is nonzero, then
\[
\Pr\!\left(D_\mu(
\widehat\Phi_{\mu,N}^{(\eta_{\rm sm})}(\zeta_N))\ge0\right)
\longrightarrow
1-\Phi_{\rm G}(\delta/\sigma_D),
\]
where \(\Phi_{\rm G}\) is the standard Gaussian cdf.
\end{proposition}

The common factor \(e^{s^2}\) in \(\Sigma_{\rm sm}\) gives the effective
context length
$N_{\rm eff}(s)=Ne^{-s^2}=N\exp(-B^2/\eta_{\rm sm})$.
Here \(e^{-s^2}\) is the population Kish fraction of the exponential
attention weights.  It determines the finite-token fluctuation scale at
fixed score and is related to attention-entropy collapse
\citep{zhai2023stabilizing}.  Averaging \(m\)
independent prompts changes token noise to
\((mN_{\rm eff})^{-1/2}\), while between-prompt switching remains
\(O(m^{-1/2})\).  Ambient leakage has scale
\(O_{\Pr}(\sqrt{(d-1)/N_{\rm eff}})\).  These statements assume fixed score
and dimension.  Appendix~\ref{app:softmax} gives
finite-horizon shadowing;
Figures~\ref{fig:softmax-finite-token}
and~\ref{fig:softmax-ambient-shadowing} test projected and full-ambient
boundary crossing, respectively.

\citet{boursier2026softmax} give non-asymptotic concentration and
population-risk gradient-flow stability for sub-Gaussian prompts.
Proposition~\ref{prop:softmax-clt} gives the explicit Gaussian-sector
covariance and discrete large-step boundary law and complements the
population analysis of \citet{goel2026training}.

\section{Experimental Results}
\label{sec:experiments}

The supplementary file contains code to regenerate every figure.  Experimental settings,
additional results, and all real-Transformer figures appear at the start of
Appendix~\ref{app:experiments}, before the proofs.  In the exact shared-mode
construction, ambient updates and
\(\Psi_{\alpha_{\mathsf B},\nu_{\mathsf B}}\) agree within
\(2.5\times10^{-15}\) over 300 random batches.
Figures~\ref{fig:crossing} and~\ref{fig:transverse-dichotomy} test,
respectively, the state-dependent concentration law and the separation of
frozen-state and pathwise transverse rates from longitudinal escape.

\paragraph{Independent prompts: quantitative approximate sharing.}
We draw independent Gaussian regression prompts with \(d=3\), a common
regressor, and signed queries \(x_{q,i}=t_i v\), without forcing empirical
Gram matrices to share an eigenvector.  For each
\(N\in\{8,16,32,64,128,256\}\), Figure~\ref{fig:approx-lsa} summarizes 80
independent 64-prompt problems and 25 ambient mini-batch steps.

\begin{figure}[t]
\centering
\includegraphics[width=\linewidth]{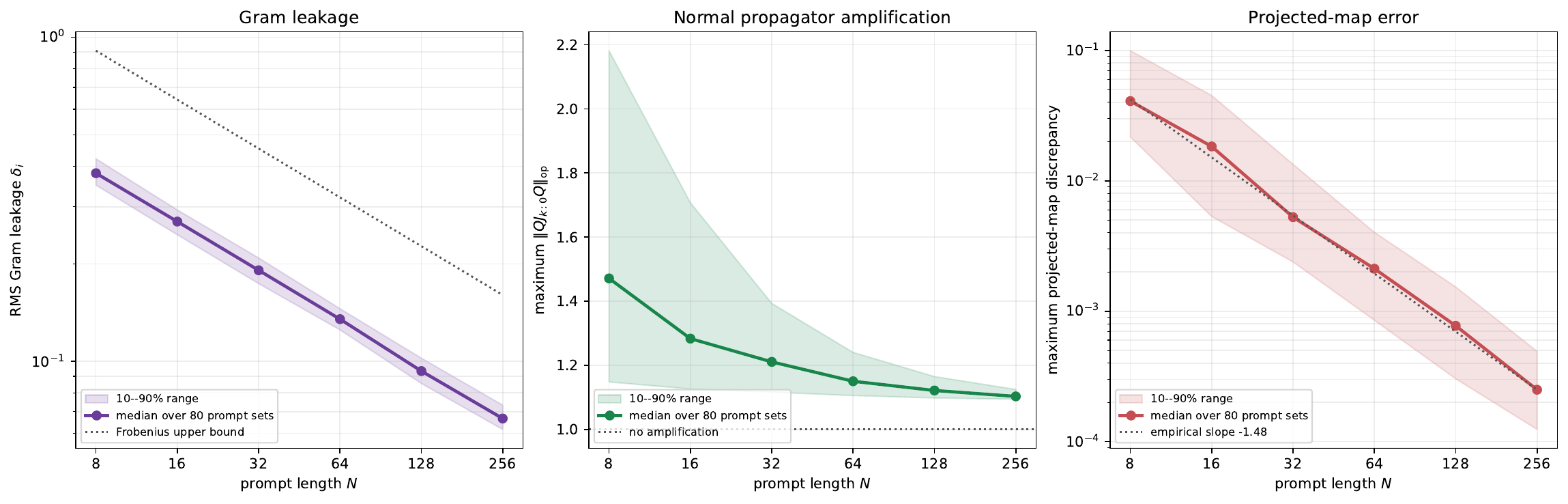}
\caption{A genuine multi-prompt approximation to the shared mode.  Training
sets are independent; only the query direction and data distribution are
shared.  RMS Gram leakage lies below the parameter-free Frobenius bound
\eqref{eq:gram-concentration} (left).  Over 25 ambient SGD steps with
\(m=8\), the origin-time normal propagator remains mildly amplifying but
moves closer to one as prompt length grows (middle), while the discrepancy
from the projected random map vanishes with prompt length (right).  Curves
show medians and 10--90\% ranges over 80 prompt sets.}
\label{fig:approx-lsa}
\end{figure}

To assess accumulation beyond the per-step proposition,
Figure~\ref{fig:approx-lsa} measures the transported leakage over 25 steps.
The median maximum normal amplification decreases from about \(1.5\) at
\(N=8\) to \(1.1\) at \(N=256\), while the projected discrepancy decreases
with prompt length.  Thus the ambient dynamics remain bounded and
increasingly close to the shared-mode approximation over the displayed
horizon.

\paragraph{Trained softmax Transformers.}
As a qualitative real-model check, four independently trained causal
softmax Transformers show the predicted finite-reference gradient-noise
scale and a replicated outward mean shift in a fixed-batch top-Hessian
curvature proxy as batch size decreases.  These trends provide a qualitative
bridge to the nonlinear outward-bias mechanism, while the exact boundary and
\(\exp(-cm)\) law remain specific tests of the reduced model.  The complete
protocol, figures, ablations, and matched rollouts are in
Appendix~\ref{app:trained-transformer-kick}.\vspace{-0.1in}

\section{Discussion}
\label{sec:discussion}
\vspace{-0.1in}

In this paper, we show that batch size changes the iterative map, its optimum, and
its stability boundary even when gradients are unbiased. We identify
batch curvature, target correlation, and state-dependent crossing
distance as useful diagnostics. A trained-Transformer counterfactual
provides a qualitative bridge to these predictions. Across four
checkpoints, smaller first-step batches yield larger deviations from a
common reference and a larger mean increase in a fixed-batch top-Hessian
curvature proxy. Under matched continuation at reference-controlled
interior rates, all four checkpoint-level slopes of median maximum
log-loss excess have the predicted sign. Although thresholded events and
individual proxy--outcome associations vary across checkpoints, the
replicated mean-bias and interior-sensitivity trends are consistent.
We also characterize balanced escape through compact common invariant
sets, leaving quasi-stationarity and off-line behavior open. Finally, we
establish finite-horizon ambient transfer with explicit control of coupled
modes, logits, and dimension. Extending \(N_{\rm eff}\) uniformly in logits
and dimension remains a natural direction for future work.

\label{main:end}

\subsection*{AI Use Statement}

In this work, we used Generative AI for formulate mathematical claims, provide critical ingredients for proving mathematical claims (e.g., carry out some calculation that were checked by the authors), assist in the writing of proofs, design or provide feedback on research  methodology or experiments (e.g., asking AI whether a specific theoretical claim can be improved toward a specific direction), implement methods (e.g., write code), literature search (e.g., ask if we missed any existing literature), language editing, proof auditing, and \LaTeX{} assistance.  The authors independently checked the mathematical arguments, citations, code, and experimental claims and take responsibility for the final manuscript.

\appendix

\section{Experimental details and additional results}
\label{app:experiments}

\subsection{Exact transformer reduction checks}

We sample a noiseless regression prompt with \(d=3\) and \(N=20\), construct
the two eigenvectors \(h_\pm\) in Proposition~\ref{prop:transformer-reduction},
and initialize the full active parameter vector in their span.  We compare
300 iterations of full active-parameter gradient descent against direct
iterations of \(\Phi_\mu\) for \(\mu\in\{0.5,0.9,1.3\}\).  No projection is
applied after initialization.  The extracted two-dimensional coordinates
agree to accumulated floating-point error.

To test the mini-batch statement itself, we take a second base prompt \(E_0\)
and form six valid regression prompts \(E_i=s_iE_0\), with
\[
    s_i\in\{0.75,0.9,1.0,1.1,1.25,1.4\}.
\]
Scaling every covariate by \(s_i\) also scales every response by \(s_i\), so
\[
    G_i=s_i^2G_0,\qquad
    X_{q,i}=s_iX_{q,0},\qquad
    H_i=s_i^3H_0,\qquad
    y_i=s_i y_0.
\]
Thus the prompts share the same \(h_\pm\), with
\(\gamma_i=s_i^3\gamma_0\), but have unequal curvature and target.  We set the
full-batch normalized parameter to \(\mu=0.55\), sample batches of size three
with replacement for 300 steps, and compare the full active-parameter update
with \(\Psi_{\alpha_{\mathsf B},\nu_{\mathsf B}}\) after every step.  The
maximum normalized discrepancy is \(2.5\times10^{-15}\).

\begin{figure}[H]
\centering
\includegraphics[width=\linewidth]{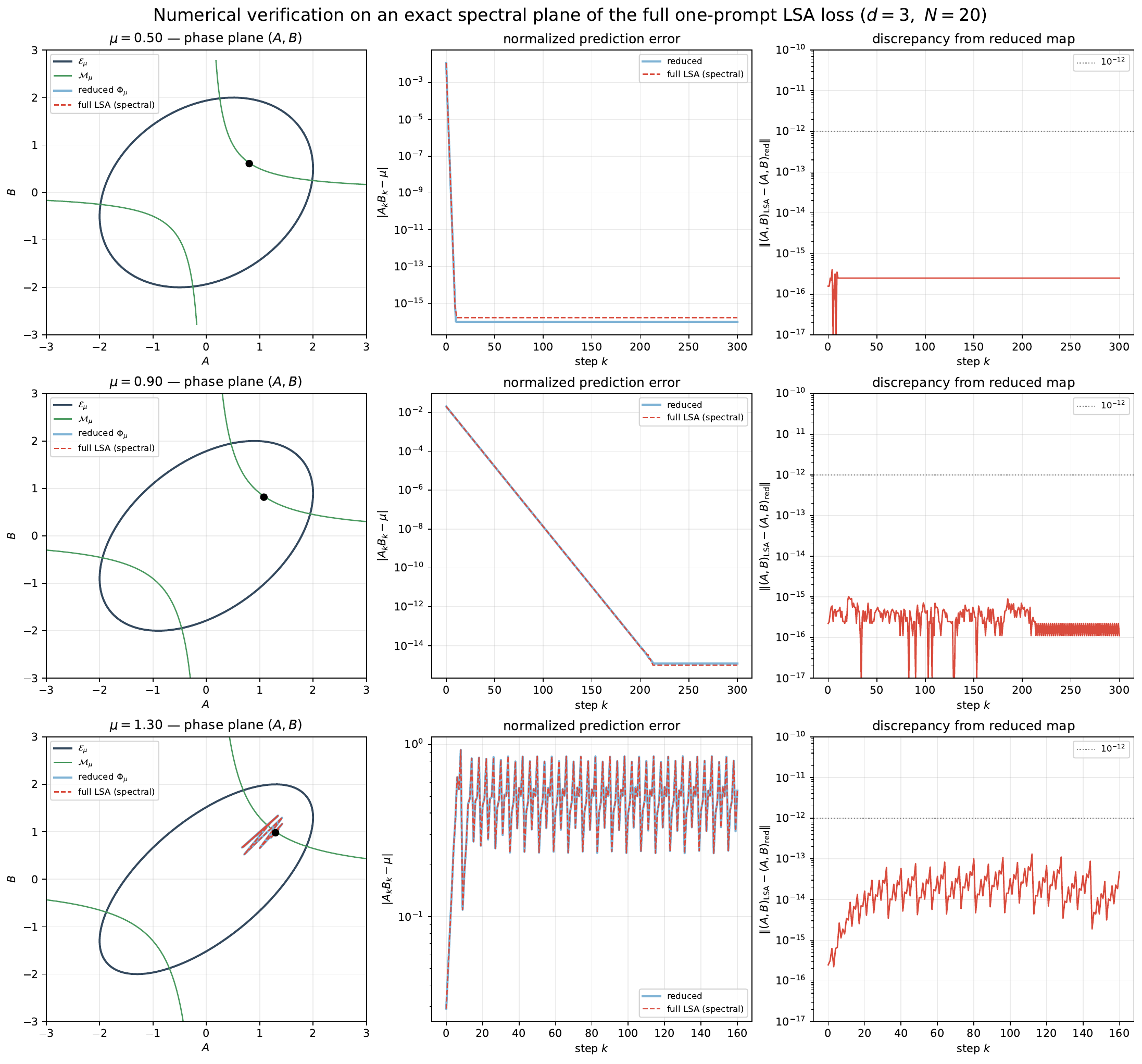}
\caption{Full active-parameter linear self-attention gradient descent and the
reduced map coincide on the invariant spectral plane.}
\label{fig:exact-reduction-app}
\end{figure}

\subsection{Separatrix crossing experiment}

For Figure~\ref{fig:crossing}, \(n=64\) losses comprise 48 low and 16 high
types.  We specify each type by its one-sample normalized parameters:
\[
(\alpha_i,\nu_i)
=
\left(\frac{\gamma_i^2}{\bar c},\,\eta\gamma_i y_i\right)
\in
\left\{
\left(0.5,-\frac4{15}\right),(2.5,3.2)
\right\}.
\]
The population means are therefore \(\E\alpha_i=1\) and
\(\E\nu_i=\mu=0.6\), while both curvature and target correlation fluctuate.
These signed types have an exact noiseless attention realization, distinct
from the fixed-regressor experiment below.  Take \(d=N=2\), training
covariates \(x_1=e_1,x_2=-e_1\), type-specific regressors
\[
w_{\rm low}=(1,-8/15),\qquad
w_{\rm high}=(1,32/25),
\]
and queries \(x_{q,i}=t_i e_2\), with
\[
t_{\rm low}=\sqrt{1/2},\qquad t_{\rm high}=\sqrt{5/2}.
\]
Both types have training labels \((1,-1)\).  For
\(p=(1,0,1)/\sqrt2\), their empirical Gram matrices satisfy \(G_i p=p\),
while the shared signed spectral values are \(\gamma_i=t_i\).  With prompt
probabilities \(3/4,1/4\) and product learning rate \(\eta=1\),
\[
\bar c=\frac34t_{\rm low}^2+\frac14t_{\rm high}^2=1,\qquad
\eta\gamma_i y_i
=t_i^2(w_i)_2
\in\{-4/15,3.2\}.
\]
Thus their normalized parameters are exactly the two displayed
\((\alpha_i,\nu_i)\).  The script verifies the corresponding ambient
one-step updates at all three plotted states to \(6.3\times10^{-16}\).

The full-batch orbit begins at
\[
z_1=(0.5,0.4),\quad
z_2=\Phi_\mu(z_1)=(0.66,0.6),\quad
z_3=\Phi_\mu^2(z_1)=(0.7824,0.73464).
\]
The left panel shows the two one-sample updates at each state.  For a batch of
size \(m\), the number of high-type samples is
\(\operatorname{Binomial}(m,1/4)\), so the exact crossing probability is a
finite binomial sum with no fitted parameters.  The right panel compares this
sum with \(2\times10^5\) sampled batches at each state and
\(m\in\{1,2,4,8,16,32,64\}\), together with the bound obtained from
\(\kappa(p)=|\rho_{\rm high}(p)-\rho_{\rm low}(p)|/2\), the exact
two-point range at that state.  Zero Monte Carlo counts are displayed at
\(1/(4\times10^5)\); the exact curves continue below that level.

\subsection{Frozen-state and pathwise transverse rates}

For Figure~\ref{fig:transverse-dichotomy}, the two prompt types are
\[
(\alpha_i,\nu_i)\in\{(0.5,0.1),(3,2.6)\}
\]
with probabilities \(0.8\) and \(0.2\).  Again
\(\E\alpha_i=1\) and \(\E\nu_i=\mu=0.6\).  At
\(\theta_\mu\), the two effective perturbations are
\(\rho_i=\nu_i-\alpha_i\mu\in\{-0.2,0.8\}\), so
\(\rho_i<1\) and Proposition~\ref{prop:frozen-vs-moment} applies.  For each
batch size, the frozen-state log and RMS rates are evaluated exactly by
summing over the binomial batch composition.

The left panel additionally measures the actual cocycle
\[
\Lambda_{25}
=\frac1{25}\sum_{k=0}^{24}
\log|1+\alpha_k r_k^2-\nu_k|
\]
over \(10^5\) balanced trajectories, using NumPy seed 314159.  We report its
distribution among paths satisfying \(D_\mu(r_k,r_k)<0\) for all 25 steps:
\begin{center}
\small
\begin{tabular}{@{}crrr@{}}
\toprule
\(m\)&\(\widehat{\Pr}(\text{survive }25)\)&
median \(\Lambda_{25}\)&
\(\widehat{\Pr}(\Lambda_{25}>0\mid\text{survive})\)\\
\midrule
1&0.641&0.0534&0.854\\
2&0.824&-0.3480&0.001\\
4&0.996&-0.1619&\(<0.001\)\\
8&\(>0.999\)&-0.0612&\(<0.001\)\\
16&1.000&-0.0260&\(<0.001\)\\
\bottomrule
\end{tabular}
\end{center}
The \(m=1\) contrast between the negative frozen value and the positive
survivor-conditioned pathwise median directly illustrates the distinction
between frozen-state and trajectory-level rates.  These measurements provide
a finite-horizon diagnostic and motivate a future invariant or
quasi-stationary interpretation.

For the right panel, we initialize \(r_0=\sqrt\mu\) and iterate
\[
r_{k+1}=r_k(1+\nu_k-\alpha_k r_k^2)
\]
for 500 steps over 4000 independent realizations at
\(m\in\{1,2,4,8,16\}\).  A realization is declared escaped when
\(|r_k|>20\) or a nonfinite value appears; once this threshold is crossed,
the cubic term makes subsequent divergence immediate at plotting scale.
Full batch is the deterministic fixed trajectory.

The script also records the realized standardized margin before the first
ellipse crossing.  Using the exact two-point range for \(\kappa\), \(K=500\),
and the fixed threshold \(\chi=1.75\), the finite-horizon diagnostic is
\begin{center}
\small
\begin{tabular}{@{}ccccc@{}}
\toprule
\(m\)&
\(\operatorname{median}\min_{k<T\wedge K}\tau_k/\kappa_k\)&
\(\widehat{\Pr}(\mathcal G_{K,\chi})\)&
\(\#\{T\le K,\mathcal G_{K,\chi}\}\)&
\(2Ke^{-m\chi^2/2}\)\\
\midrule
8&1.799&0.729&0/4000&\(4.79\times10^{-3}\)\\
16&1.799&\(>0.999\)&0/4000&\(2.29\times10^{-8}\)\\
\bottomrule
\end{tabular}
\end{center}
Thus Theorem~\ref{thm:sequential-crossing} tightly controls the joint event
of direct crossing and retained margin.  The frequent \(m\le4\) escapes occur
after the measured margin has eroded, cleanly separating direct batch
crossings from state evolution and identifying margin dynamics as the next
ingredient for a sharper escape-time analysis.

\subsection{Approximate shared mode with independent prompts}

For Figure~\ref{fig:approx-lsa}, we use \(d=3\),
\(w_\star=(1,0.5,-0.3)\), standard Gaussian training covariates, and query
direction \(v=w_\star/\|w_\star\|\).  Each of the 64 prompts has an
independent training set and query
\[
x_{q,i}=t_i v,\qquad
|t_i|\sim\operatorname{Unif}[0.8,1.2],
\]
with an independent random sign.  The vector \(p\) is the top eigenvector of
the population Gram
\[
\bar G=\frac12
\begin{pmatrix}I&w_\star\\w_\star^\top&\|w_\star\|^2\end{pmatrix}.
\]
For each \(N\in\{8,16,32,64,128,256\}\), we generate 80 independent
64-prompt collections.  The product learning rate is normalized to
\(\mu=0.55\), the reduced initialization is \((A_0,B_0)=(0.70,0.55)\), and
we run 25 steps with batch size \(m=8\), using the same realized batches for
the ambient and reduced trajectories.

The reported Gram leakage is
\(\sqrt{n^{-1}\sum_i\delta_i^2}\).  Let \(Q=I-P\), and let \(J_k\) be the
Jacobian of the realized ambient batch update along the ambient trajectory.
The middle panel reports
\[
\max_{j\le25}\|QJ_{j-1}\cdots J_0Q\|_{\rm op},
\]
the worst cumulative amplification of an initially normal perturbation.
The right panel is the maximum distance between the projected ambient
coordinates and the paired
\(\Psi_{\alpha_{\mathsf B},\nu_{\mathsf B}}\) trajectory.  No projection is
applied to ambient SGD\@.  The dotted curve in the left panel is the
parameter-free bound
\[
\left[
\frac{(\operatorname{tr}C)^2+\|C\|_F^2}{4N}
+\frac{\E|t_i|^4}{4N^2}
\right]^{1/2},
\]
which follows from \eqref{eq:gram-concentration} and
\(\delta_i\le\|G_i-\bar G\|_F\); it uses no fitted anchoring.  The normal
propagator remains bounded and moves toward neutral amplification: its
median maximum norm decreases from about \(1.5\) at \(N=8\) to \(1.1\) at
\(N=256\).  The script also records maximum off-plane norm, whose median has
empirical slope about \(-0.92\).  Centered leakage cancels across changing
batches, the coherent query atom is \(O(N^{-1})\), and projected feedback is
second order here, explaining why projected error can decay faster than the
per-step \(N^{-1/2}\) bound.  Together these observations support the
finite-horizon shared-mode approximation even when the normal propagator is
mildly amplifying.

\subsection{Real-data ablations}
\label{app:trained-transformer-kick}

\paragraph{Model, data, and checkpoints.}
We train a pre-LayerNorm causal language model with two softmax-attention
blocks, four heads, width 128, MLP ratio 4, tied input/output embeddings, and
context length 128.  The model has \(1{,}434{,}880\) trainable parameters.
WikiText-2 is encoded with an 8000-entry byte-level BPE tokenizer trained on
the training split, yielding 2,738,892 training tokens and 288,741 validation
tokens.  Training uses unmodified SGD, mini-batch size 32, learning rate
\(0.1\), and a 100-step linear warmup.  We use the step-1800 checkpoints from
independent seeds \(17,29,41,53\), whose fixed-set validation losses are
\(6.2631,6.1122,6.1231,6.1374\), respectively.  Reloading every archived
checkpoint reproduces its stored loss within the fixed numerical tolerance.
Seeds 17 and 29 originated in the pilot; seeds 41 and 53 were trained after
the multi-checkpoint protocol was fixed.

\paragraph{Reference-only calibration.}
For each checkpoint, we first run a 100-step all-reference continuation on
the fixed grid
\[
  \eta\in\{0.10,0.15,\ldots,0.55\}.
\]
The first update uses a reference gradient averaged over 32 microbatches of
size 32, so \(M=1024\); subsequent updates use the same first 16
microbatches.  Evaluation always uses the same three validation batches of
size 24.  A grid rate fails if the trajectory does not complete 100 steps,
becomes nonfinite, or exceeds validation loss 50.  We define the
\emph{pre-failure rate} as the grid point immediately preceding the first
failure and the \emph{interior rate} as the largest smaller grid point whose
100-step reference maximum is at most \(\log8000=8.987\).  Seed 53 had no
failure on the initial grid, so, before evaluating any candidate kick, we
fixed a contingency extending the same spacing through \(1.0\); its first
failure occurs at \(0.65\).  The complete scan is retained because survival
need not be monotone in a large probe step.

\begin{center}
\small
\begin{tabular}{@{}ccccccc@{}}
\toprule
Seed&Checkpoint loss&\(\eta_{\rm int}\)&Int.\ max&
\(\eta_{\rm pre}\)&Pre-failure max&First failure\\
\midrule
17&6.263&0.30&8.511&0.35&8.892&0.40\\
29&6.112&0.15&8.292&0.35&12.111&0.40\\
41&6.123&0.25&8.729&0.50&16.191&0.55\\
53&6.137&0.20&8.503&0.60&29.384&0.65\\
\bottomrule
\end{tabular}
\end{center}

\paragraph{Matched intervention and curvature proxy.}
At each selected rate and each
\(m\in\{8,32,128\}\), we generate 128 candidate first batches.  Thirty-two
indices are selected uniformly using seed 31415; the same indices are used
at the interior and pre-failure rates.  After the first candidate update,
every branch receives the same 12-step continuation, using the same future
reference batches, step size, and evaluation set as its all-reference
trajectory.  The primary finite-horizon response is
\[
  Y_m=\max_{0\le t\le12}
  \log\frac{L_{\rm candidate}(t)}{L_{\rm reference}(t)}.
\]

For every selected post-update state and its reference state, we estimate
the top Hessian eigenvalue of cross-entropy on one independent fixed training
batch of 32 sequences.  The Hessian batch is sampled with seed 91003 and is
disjoint from the intervention and continuation batches.  We use 20-step
Lanczos with two-pass full reorthogonalization and fixed initialization seed
93001.  All \(4\times2\times(3\times32+1)=776\) estimates are finite; the
maximum relative Ritz residual is \(3.12\times10^{-5}\).  Our
reference-normalized local curvature coordinate is
\[
  C_m=
  \frac{\lambda_{\max}(\theta-\eta g_m)
       -\lambda_{\max}(\theta-\eta\bar g_M)}
       {|\lambda_{\max}(\theta-\eta\bar g_M)|}.
\]
The normalization removes the mechanical multiplication by the different
probe rates.  We interpret \(C_m\) as a practical local stability proxy; the
exact boundary coordinate and classical threshold
\(\eta\lambda_{\max}=2\) remain properties of the reduced setting rather
than assumptions imposed on this nonconvex fixed-batch Hessian.

This relative normalization retains the meaningful step-size dependence:
\(\eta\) determines the candidate and reference post-update states and their
subsequent dynamics.  We therefore compare batch-size effects within each
fixed calibrated rate, preserving the state-dependent character emphasized
by the theory.

\begin{figure}[H]
\centering
\includegraphics[width=\linewidth]{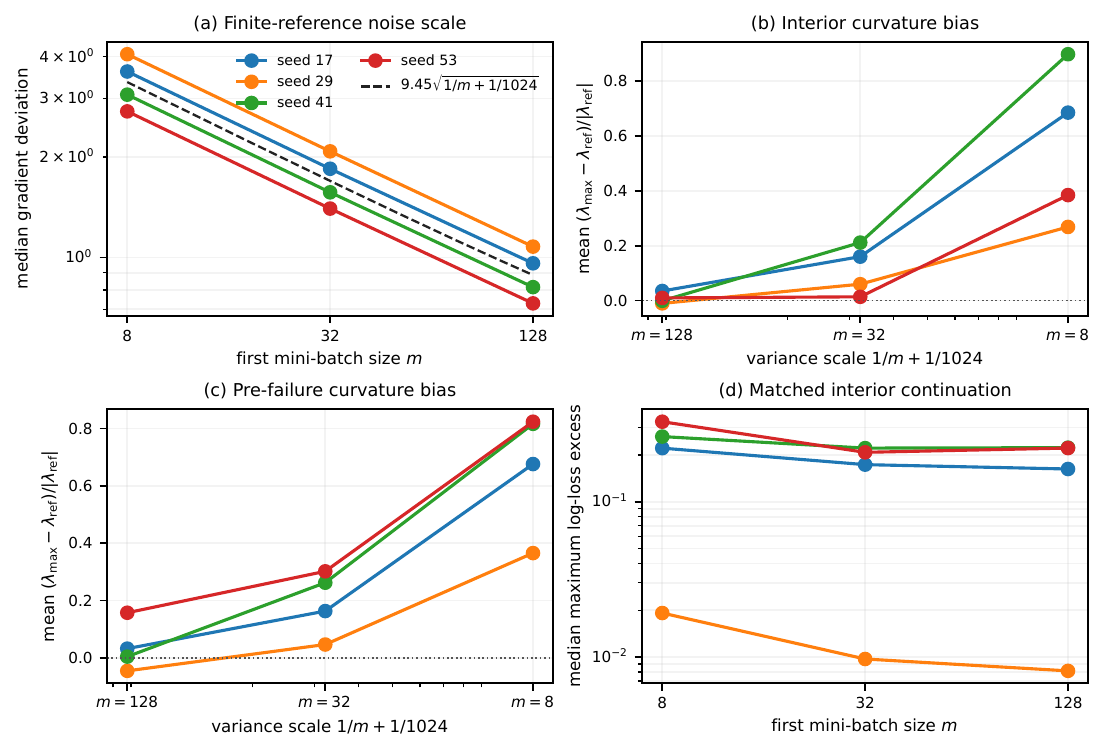}
\caption{Real-data ablations across four independently trained
softmax Transformers.  (a) For 128 candidates per checkpoint and size, the
median gradient deviation follows the finite-reference scale; the dashed
guide uses the median fitted coefficient across checkpoint--size cells.
(b,c) Mean reference-normalized top-curvature shift over the 32 uniformly
selected candidates.  Every checkpoint has positive slope against
\(1/m+1/1024\), both at the interior and pre-failure rates.
(d) Under the matched continuation at the reference-controlled interior
rates, every checkpoint has positive least-squares slope of median maximum
log-loss excess against \(1/m+1/1024\).  Lines identify training seeds; the
same colors are used in every panel.}
\label{fig:transformer-theory-ablations}
\end{figure}

\paragraph{Ablation 1: the finite-reference fluctuation scale.}
Let
\[
  R_m=\frac{\|g_m-\bar g_M\|}{\|\bar g_M\|}.
\]
For independent per-sequence gradients with covariance \(\Sigma\),
\[
  \E\|g_m-\bar g_M\|^2
  =
  \left(\frac1m+\frac1M\right)\operatorname{tr}\Sigma.
\]
The two estimation variances add because their cross-term has zero
expectation.  Within each checkpoint, dividing the medians of \(R_m\) by
\(\sqrt{1/m+1/1024}\) is nearly constant across the three sizes:
\[
\begin{array}{c|ccc}
&m=8&m=32&m=128\\ \hline
\text{seed 17}&10.171&10.279&10.251\\
\text{seed 29}&11.460&11.591&11.503\\
\text{seed 41}& 8.669& 8.737& 8.698\\
\text{seed 53}& 7.725& 7.810& 7.775
\end{array}
\]
The identity concerns second moments, and the nearly constant rescaled
medians provide its empirical finite-reference counterpart.  This verifies
the fluctuation scale used in the reduced theory while retaining the \(1/M\)
contribution omitted by an infinite-reference idealization.

\paragraph{Ablation 2: nonlinear outward curvature bias.}
Ordinary averaging predicts the size of \(g_m-\bar g_M\), but not the sign of
its effect on a nonlinear curvature coordinate.  The mean curvature shifts
in Figure~\ref{fig:transformer-theory-ablations} are
\[
\begin{array}{c|ccc}
&m=8&m=32&m=128\\ \hline
\text{interior}&0.558&0.111&0.0086\\
\text{pre-failure}&0.670&0.193&0.037
\end{array}
\]
where each entry averages the four checkpoint-level group means.  More
importantly, all eight checkpoint--rate least-squares slopes against
\(1/m+1/1024\) are positive:
\[
\begin{array}{c|cc}
\text{seed}&\text{interior slope}&\text{pre-failure slope}\\ \hline
17&5.553&5.487\\
29&2.330&3.476\\
41&7.558&6.637\\
53&3.409&5.652
\end{array}
\]
Treating checkpoints, rather than individual first batches, as the
replication unit, the mean slopes are \(4.71\) with 95\% checkpoint-bootstrap
interval \([2.87,6.56]\) in the interior and \(5.31\) with interval
\([4.02,6.35]\) pre-failure.  This is the closest real-model analogue of the
conditional drift identity in the exact model: finite-batch variation is
centered at the update level, while evaluation through a nonlinear local
stability proxy produces a signed mean shift.  We interpret \(C_m\) as an
empirical analogue rather than a literal copy of \(D_\mu\).  Both calibrated
regimes independently show positive slopes; their difference is a secondary
question at the present checkpoint resolution.

\paragraph{Ablation 3: matched interior continuation.}
The interior calibration directly addresses the concern that a reference may
merely be slower to diverge: all four 100-step reference maxima lie below
uniform-prediction loss.  At these rates, the checkpoint-level slopes of the
median response \(Y_m\) against \(1/m+1/1024\) are
\[
  0.506,\qquad0.096,\qquad0.369,\qquad1.010
\]
for seeds \(17,29,41,53\).  Their mean is \(0.495\), with 95\%
checkpoint-bootstrap interval \([0.199,0.850]\).  Descriptively pooling the
equal-sized checkpoint cohorts gives median maximum log-loss excesses
\(0.262,0.172,0.157\) for \(m=8,32,128\).  Thus a first batch can have a
small immediate effect yet leave a measurable finite-horizon difference
under matched future data.  This finite-horizon sensitivity result
complements the exact crossing exponent derived in the reduced model.

\paragraph{Complementary aggressive-rate illustration.}
The original seed-29 experiment preceded the multi-checkpoint protocol and is
retained as an illustration of a large-loss realization of the mechanism.
Starting from the same checkpoint, it compares 512 candidates per size with
the \(M=1024\) reference at \(\eta=0.4\), then follows 64 uniformly selected
candidates per size for 12 matched steps.  The reference stays below loss 50
over those 12 steps but reaches \(9.875\), above \(\log8000\), and becomes
nonfinite at step 34, providing a deliberately aggressive finite-horizon
comparator.

\begin{figure}[H]
\centering
\includegraphics[width=\linewidth]{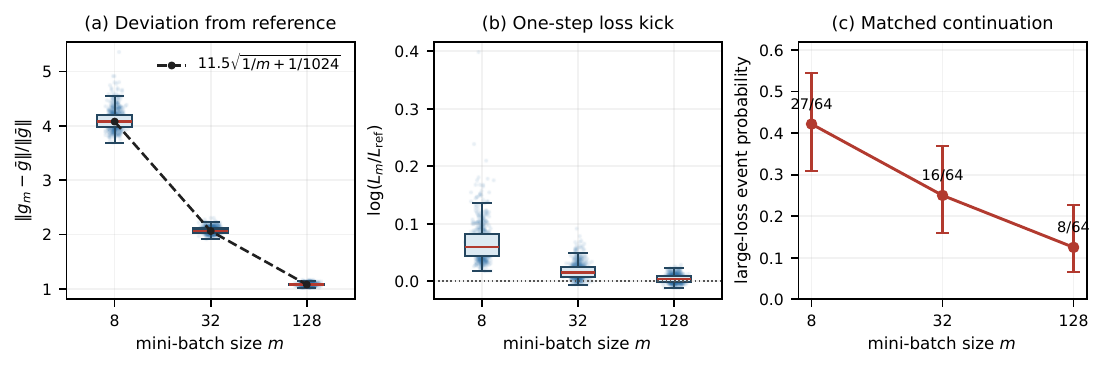}
\caption{Complementary finite-batch kicks from the seed-29 checkpoint at probe
rate \(0.4\).  (a) Relative candidate-gradient deviation from the
\(M=1024\) reference; the dashed guide includes the finite-reference term.
(b) One-step log validation-loss excess relative to the reference update.
(c) Loss-50 event frequency among 64 uniformly selected first batches, with
Wilson 95\% intervals.  This discovery experiment motivated the subsequent
four-checkpoint curvature analysis and is reported separately.}
\label{fig:trained-transformer-kick}
\end{figure}

The medians of \(R_m\) are \(4.073,2.065,1.080\), and the 90th percentiles of
the one-step log-loss excess are \(0.1105,0.0374,0.0135\).  The matched
loss-50 frequencies are
\[
  27/64=42.2\%,\qquad16/64=25.0\%,\qquad8/64=12.5\%.
\]
Their Wilson intervals are \([30.9,54.4]\%\), \([16.0,36.8]\%\), and
\([6.5,22.8]\%\).  A Cochran--Armitage score test with score \(\log_2m\)
gives two-sided \(p=1.43\times10^{-4}\), and the batch-8 to batch-128 risk
ratio is \(3.38\) with 95\% interval \([1.66,6.86]\).  Reclassifying the same
cohorts at thresholds \(20,100,500,1000\) preserves the ordering, with event
counts \(36/30/13\), \(22/11/8\), \(15/7/5\), and \(14/6/3\).

\begin{figure}[H]
\centering
\includegraphics[width=\linewidth]{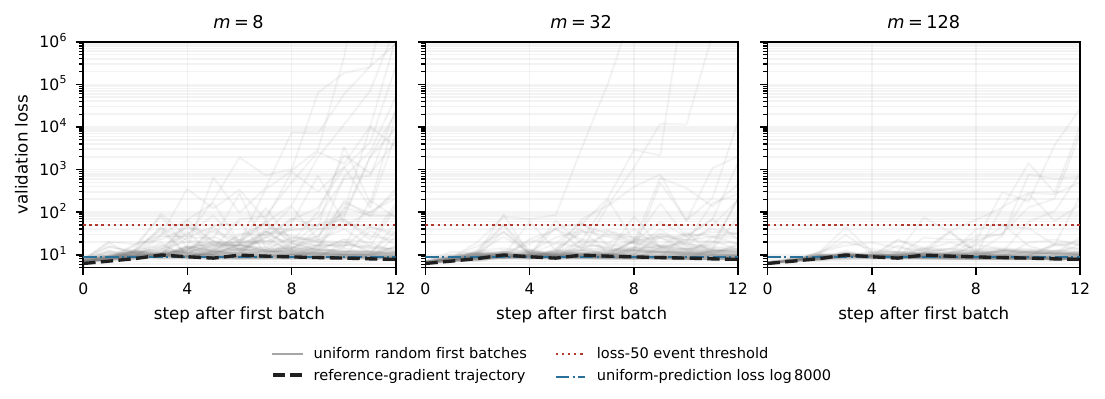}
\caption{Complementary matched trajectories following the seed-29 first-batch
kick.  Every gray trajectory uses a uniformly selected first batch and then
the common reference-gradient rule; the dashed black trajectory uses the
reference gradient from the first update onward.  The red dotted and blue
dash-dot lines mark loss 50 and \(\log8000\).}
\label{fig:trained-transformer-rollouts}
\end{figure}

A state-control run at \(\eta=0.15\) gives a 100-step reference maximum
\(8.292\) and final loss \(6.157\).  In a 50-step screen with 16 uniformly
selected first batches per size, every branch remains below loss 15.  This
safe-interior control is the state-dependent counterpart of the aggressive
probe: far enough from reference failure, the same first-batch perturbations
remain inside the observed finite-horizon safe region.

\begin{figure}[H]
\centering
\includegraphics[width=\linewidth]{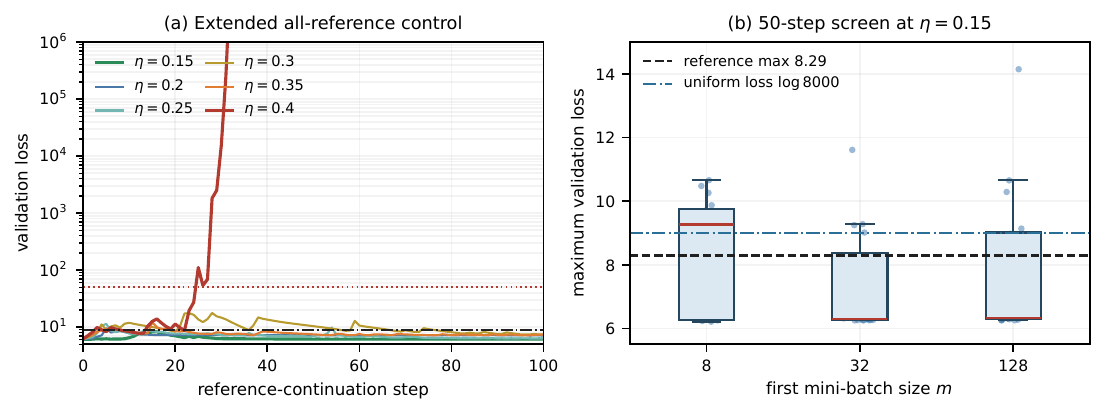}
\caption{Seed-29 state control.  (a) The \(\eta=0.4\) reference becomes
nonfinite at step 34, whereas the \(\eta=0.15\) reference remains below
uniform-prediction loss for 100 steps.  (b) Maximum loss over 50 matched
steps at \(\eta=0.15\); no uniformly selected branch exceeds 15.}
\label{fig:trained-transformer-reference-control}
\end{figure}

\paragraph{Interpretation.}
The manuscript foregrounds the predictions whose sign replicates across all
four checkpoints: the finite-reference fluctuation scale, the outward mean
shift in a nonlinear curvature proxy, and interior matched-continuation
sensitivity.  Preregistered secondary thresholded events and individual
proxy--outcome correlations vary across checkpoints, providing complementary
evidence about state dependence while leaving the replicated mean trends
intact.  The controlled use of one architecture scale, dataset, optimizer,
evaluation set, and future batch sequence makes this study a focused
qualitative bridge to the exact mechanism.  Identifying an invariant plane
or directly testing the \(\exp(-cm)\) law in a larger model are promising
extensions.

\subsection{Finite-token softmax boundary layer}

Figure~\ref{fig:softmax-finite-token} tests
Proposition~\ref{prop:softmax-clt} directly.  In the first two panels,
\(Z_j=y_j/\|w_\star\|\stackrel{\rm iid}{\sim}\mathcal N(0,1)\) are the
normalized prompt labels in the ICL sector above.  We use
\[
N\in\{64,128,256,512,1024,2048\},
\qquad s\in\{0,0.6,1\}.
\]
Solid curves are root-mean-square errors over 6000
independent prompts; dotted curves are the parameter-free predictions
\[
\sqrt{\frac{e^{s^2}(1+s^2)}{N}},
\qquad
\sqrt{\frac{e^{s^2}(s^4+4s^2+2)}{N}}.
\]
At \(s=1\), the tilted variance approaches its asymptotic curve through a
visible finite-\(N\) regime induced by the lognormal weights, with no fitted
correction.

For the right panel, \(\mu=0.6\), \(\eta_{\rm sm}=1\), and the limiting
boundary state is the ellipse point
\[
\zeta_\star
=
\left(
\sqrt{2+\mu}\cos(0.85)+\sqrt{2-\mu}\sin(0.85),\
\sqrt{2+\mu}\cos(0.85)-\sqrt{2-\mu}\sin(0.85)
\right).
\]
For each \(N\) and normalized margin
\(r\in\{0,0.5,1\}\), we scale this point inward to obtain \(\zeta_N\)
satisfying
\[
\sqrt N\,D_\mu(\Phi_\mu(\zeta_N))=-r\sigma_D.
\]
We then draw 12,000 independent prompts, apply the exact map
\eqref{eq:finite-softmax-map}, and report the fraction landing at
\(D_\mu\ge0\).  The dotted limits are
\(1-\Phi_{\rm G}(r)\), with no fitted parameters.

\begin{figure}[H]
\centering
\includegraphics[width=\linewidth]{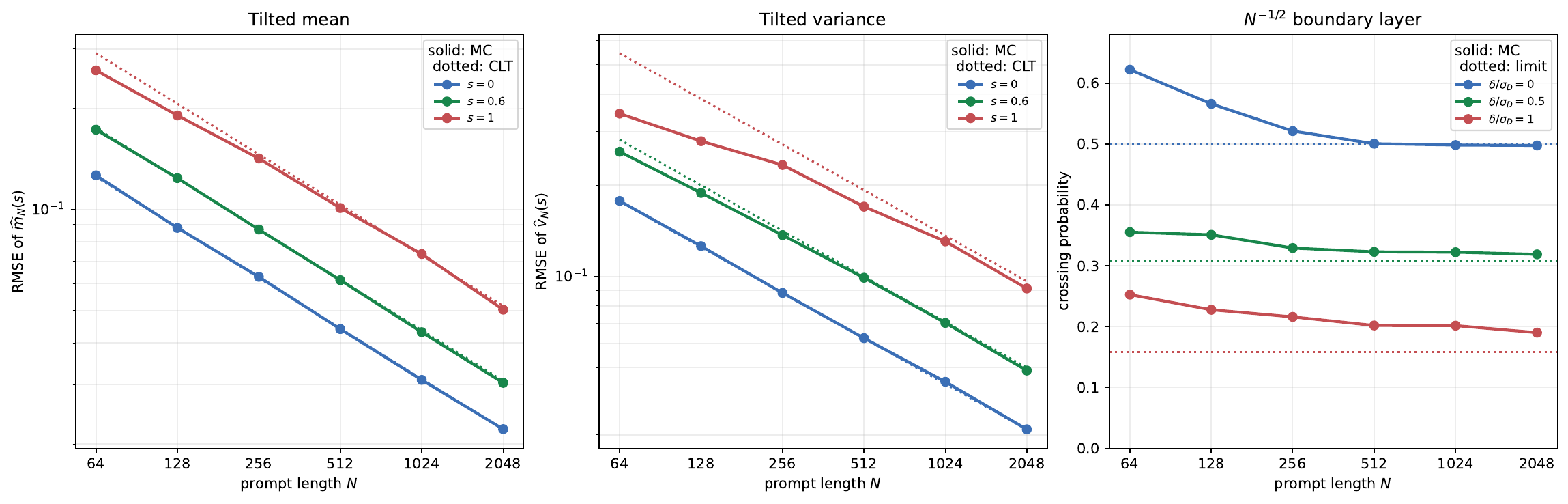}
\caption{Finite-token softmax fluctuations.  Left and middle: empirical
errors of the tilted mean and variance approach the explicit Gaussian
central-limit predictions.  Right: when the population image lies within
an \(N^{-1/2}\) boundary layer, the crossing probability converges to the
Gaussian limit in Proposition~\ref{prop:softmax-clt}; at the displayed
context lengths, the curves approach the limit from above, consistently
with lognormal skew.}
\label{fig:softmax-finite-token}
\end{figure}

\subsection{Ambient finite-token softmax boundary crossing}

We now test the boundary law with a full ambient value row and query--key
matrix.  The setup uses \(d=3\), \(w_\star=(1,0.5,-0.3)\),
\(\eta_{\rm sm}=1\), \(\mu=0.6\), and exactly the boundary-layer states
\(\zeta_N\) used in Figure~\ref{fig:softmax-finite-token}.  For each
\(N\in\{64,128,256,512,1024,2048\}\) and
\(\delta/\sigma_D\in\{0,0.5,1\}\), we draw 8000 labeled Gaussian prompts,
take one full ambient gradient step, and only then project its endpoint to
evaluate \(D_\mu\).

Besides the exact sector initialization, we take
\[
u=(u_x,0),\qquad u_x=\frac{(1,-2,0)}{\sqrt5},
\qquad
\xi_N=\frac{1}{c\sqrt{2h\eta_{\rm sm}N}},
\]
and initialize
\[
\mathcal A_{0,N}=\alpha_Nr_\star^\top+\xi_Nu^\top,\qquad
Q_{0,N}=\beta_Nvr_\star^\top+\xi_Nvu^\top.
\]
Here \(u_x\perp w_\star\), and the exact normalization is
\[
\sqrt{\eta_{\rm sm}}c\sqrt h
\left(
\|\xi_Nu\|^2+\|\xi_Nvu^\top\|_F^2
\right)^{1/2}
=N^{-1/2}.
\]
Thus the perturbation matches the finite-token leakage scale.
Lemma~\ref{lem:softmax-paired-perturbation} gives, for the same empirical
prompt, an \(O_{\Pr}(N^{-1})\) difference between the two projected
endpoints.  This is negligible relative to the \(N^{-1/2}\) fluctuation in
Proposition~\ref{prop:softmax-clt}, so the first-order crossing limit is
unchanged.

\begin{figure}[H]
\centering
\includegraphics[width=\linewidth]{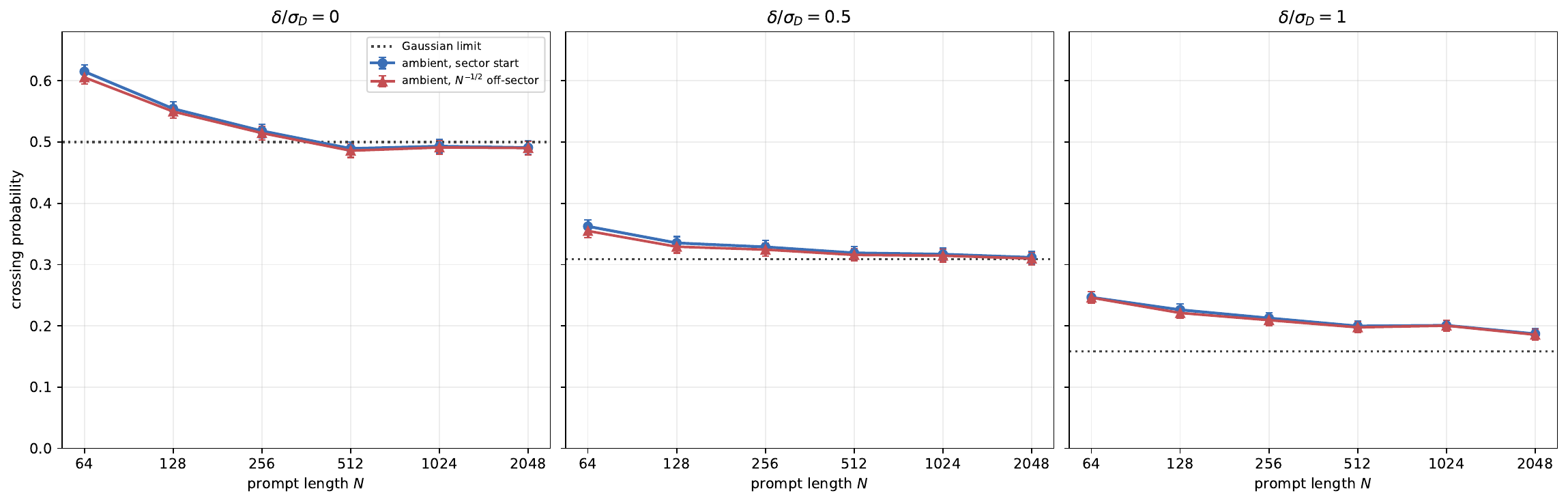}
\caption{Near-boundary full ambient softmax updates.  Each panel fixes the
normalized margin shown above it.  Circles start exactly in the rank-one
sector; triangles add an \(N^{-1/2}\) normalized off-sector perturbation.
Both use full ambient gradients, with projection only after the step to
evaluate crossing.  Bars are 95\% binomial intervals over 8000 paired
prompts; dotted lines are the parameter-free Gaussian limits from
Proposition~\ref{prop:softmax-clt}.}
\label{fig:softmax-ambient-shadowing}
\end{figure}

For an exact-sector start, projection of the ambient endpoint agrees with
the finite-token scalar map to \(1.9\times10^{-15}\).  Across the displayed
settings, the boundary-scale perturbation changes at most \(5.8\%\) of paired
crossing decisions, and both curves approach the same Gaussian limit.  This
provides a direct one-step ambient validation; the multi-step diagnostic
below complements it by tracking subsequent leakage and amplification.

\subsection{Multi-step ambient softmax leakage}

The previous experiment is deliberately local to one boundary crossing.  A
complementary trajectory diagnostic starts well inside the ellipse and
tracks repeatedly injected leakage after ambient training has left the
sector.  We use \(d=3\), \(w_\star=(1,0.5,-0.3)\),
\(\eta_{\rm sm}=0.5\), \(\mu=0.6\), and normalized initialization
\((A_0,B_0)=(0.70,0.55)\).  Each of 25 steps draws a fresh labeled Gaussian
prompt, and the paired scalar trajectory uses the same tokens.  For each
\(N\in\{32,64,128,256,512,1024\}\), we run 80 realizations.

\begin{figure}[H]
\centering
\includegraphics[width=0.76\linewidth]{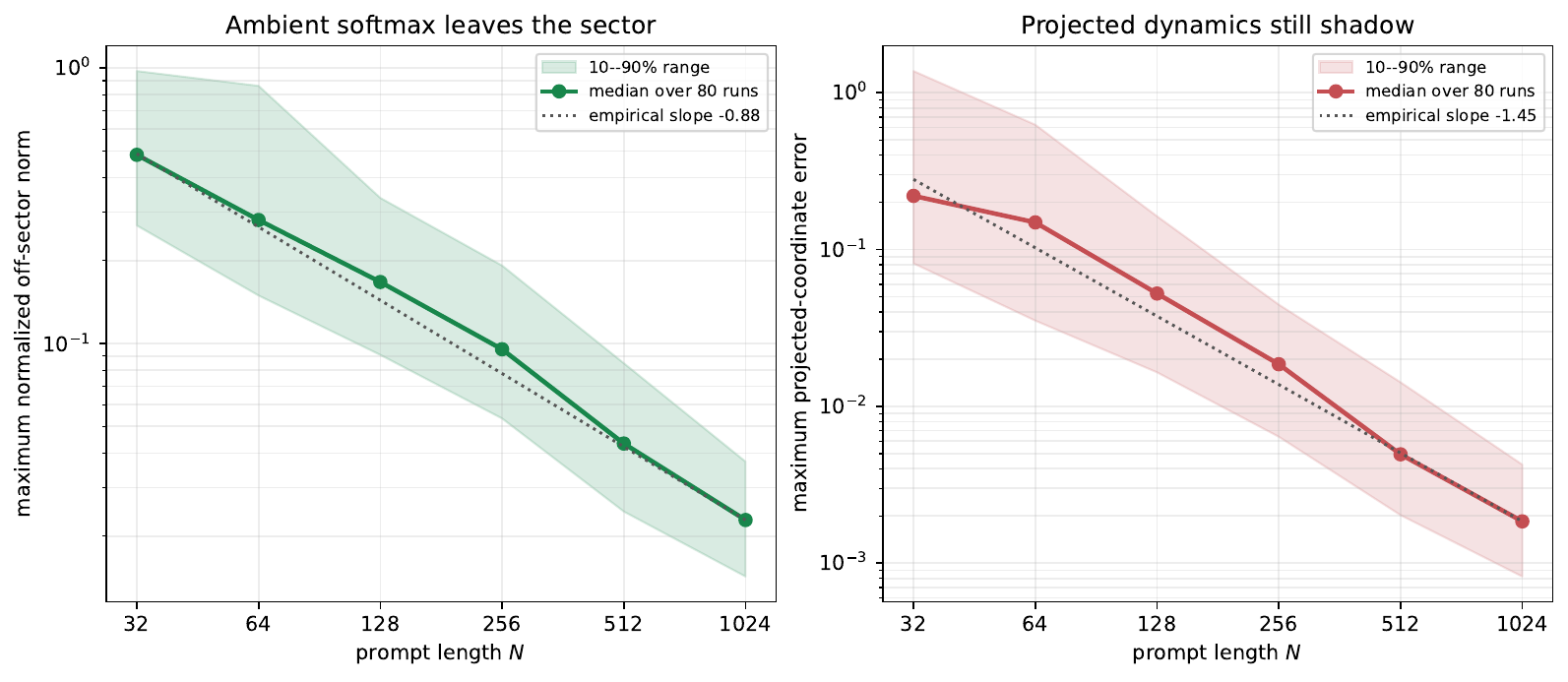}
\caption{Complementary 25-step full ambient softmax diagnostic (no projection
after initialization).  Left: maximum normalized off-sector norm.  Right:
maximum projected-coordinate discrepancy from the paired scalar trajectory.
Curves show medians and 10--90\% ranges over 80 runs; both decay with \(N\).}
\label{fig:softmax-ambient-interior-shadowing}
\end{figure}

This complements Figure~\ref{fig:softmax-ambient-shadowing} by documenting
finite-horizon shadowing away from the boundary.  Together the two
experiments cover both a local crossing event and short ambient trajectories.

\section{Proof of the transformer reduction}
\label{app:transformer-proof}

\begin{proof}[Proof of Proposition~\ref{prop:transformer-reduction}]
Let \(\rho_q=\|x_q\|\).  The query-dependent matrix in the active
linear-self-attention loss has the form
\[
X_q=
\begin{pmatrix}
0_{d\times d}&x_q\\
x_q^\top&0
\end{pmatrix}.
\]
Its two nonzero eigenpairs are
\[
X_qr_\pm=\pm\rho_qr_\pm,
\qquad
r_\pm=\frac1{\sqrt2}
\begin{pmatrix}
x_q/\rho_q\\
\pm1
\end{pmatrix}.
\]
If \(Gp=\gamma p\), then for \(H=X_q\otimes G\),
\[
h_\pm=r_\pm\otimes p,
\qquad
Hh_\pm=\pm\lambda h_\pm,
\qquad
\lambda=\rho_q\gamma.
\]

Write the loss as
\[
\mathcal L(u)=\frac12(u^\top Hu-y_q)^2
\]
and initialize
\[
u_k=c_{+,k}h_++c_{-,k}h_-.
\]
Since
\[
\nabla\mathcal L(u)=2(u^\top Hu-y_q)Hu,
\]
the span of \(h_+\) and \(h_-\) is invariant under every finite gradient
step.  Set
\[
a_k=c_{+,k}+c_{-,k},
\qquad
b_k=c_{+,k}-c_{-,k}.
\]
Orthogonality gives
\[
u_k^\top Hu_k
=\lambda(c_{+,k}^2-c_{-,k}^2)
=\lambda a_kb_k.
\]
If \(E_k=\lambda a_kb_k-y_q\), the coefficient updates imply
\[
a_{k+1}=a_k-2\eta\lambda E_kb_k,
\qquad
b_{k+1}=b_k-2\eta\lambda E_ka_k.
\]

Let
\[
\sigma=\sgn(y_q),
\qquad
\theta=2\eta\lambda^2,
\qquad
A_k=\sqrt\theta\,a_k,
\qquad
B_k=\sigma\sqrt\theta\,b_k.
\]
Then
\[
E_k
=\frac{\sigma\lambda}{\theta}(A_kB_k-\mu),
\qquad
\mu=\frac{\theta|y_q|}{\lambda}
=2\eta\lambda|y_q|.
\]
Substitution yields
\[
A_{k+1}=A_k-(A_kB_k-\mu)B_k,
\qquad
B_{k+1}=B_k-(A_kB_k-\mu)A_k.
\]
Relabeling \(A,B\) as \(a,b\) proves \eqref{eq:phi}.
\end{proof}

\begin{proof}[Proof of Proposition~\ref{prop:shared-mode}]
Write
\[
u=c_+h_++c_-h_-,
\qquad
a=c_++c_-,
\qquad
b=c_+-c_-.
\]
For every prompt \(i\), the shared plane is invariant under \(H_i\), and
\[
u^\top H_i u
=\gamma_i(c_+^2-c_-^2)
=\gamma_iab.
\]
Since
\[
\nabla\mathcal L_i(u)
=2(u^\top H_i u-y_i)H_i u,
\]
each individual prompt gradient lies in the same plane.  Averaging over a
batch and updating with learning rate \(\eta_{\rm tr}\) gives
\begin{align*}
c_+^+
&=c_+-2\eta_{\rm tr}
\left[\frac1m\sum_{i\in\mathsf B}
\gamma_i(\gamma_iab-y_i)\right]c_+,\\
c_-^+
&=c_-+2\eta_{\rm tr}
\left[\frac1m\sum_{i\in\mathsf B}
\gamma_i(\gamma_iab-y_i)\right]c_-.
\end{align*}
Adding and subtracting these equations proves the displayed update in
Proposition~\ref{prop:shared-mode}.  The bracketed factor is the gradient
coefficient of the batch average of \eqref{eq:sample-loss}; the change of
coordinates contributes the factor \(2\), so its effective product-model
learning rate is \(\eta=2\eta_{\rm tr}\).
\end{proof}

\section{Approximate shared modes for independent prompts}
\label{app:approx-mode-proof}

\begin{proof}[Proof of Proposition~\ref{prop:approx-shared-mode}]
Because \(x_{q,i}=t_i v\) and \(\|v\|=1\),
\[
X_{q,i}r_\pm=\pm t_i r_\pm.
\]
Decompose
\[
G_i p=g_i p+d_i,\qquad
d_i=(I-pp^\top)G_i p,\qquad \|d_i\|=\delta_i.
\]
For \(u=c_+h_++c_-h_-\), orthogonality of \(r_+\), \(r_-\), \(p\), and
\(d_i\) gives
\begin{align*}
H_i u
&=
t_i g_i(c_+h_+-c_-h_-)
+t_i\bigl(c_+r_+\otimes d_i-c_-r_-\otimes d_i\bigr),\\
u^\top H_i u
&=t_i g_i(c_+^2-c_-^2)
=\gamma_iab.
\end{align*}
The first term is the projection of \(H_i u\) onto \(\mathcal S\), while
the second is orthogonal to \(\mathcal S\) and has norm
\(|t_i|\delta_i\sqrt{c_+^2+c_-^2}=|t_i|\delta_i\|u\|\).  Multiplying by
the residual factor \(2(\gamma_iab-y_i)\) proves
\eqref{eq:gradient-leakage} and the exact projected product gradient.

For the Gaussian statement, let
\[
S_N=\frac1N\sum_{j=1}^N z_jz_j^\top,\qquad
\E[z_jz_j^\top]=C.
\]
Isserlis' identity yields
\[
\E\|S_N-C\|_F^2
=\frac1N\sum_{a,b}\operatorname{Var}(z_a z_b)
=\frac{(\operatorname{tr}C)^2+\|C\|_F^2}{N}.
\]
The empirical prompt Gram is
\[
G_i=\frac12S_N+\frac1{2N}q_iq_i^\top.
\]
Since \(\E(S_N-C)=0\), the cross term vanishes after expectation, and
\[
\E\|G_i-C/2\|_F^2
=\frac14\E\|S_N-C\|_F^2+\frac{\|q_i\|^4}{4N^2},
\]
which is \eqref{eq:gram-concentration}.  Finally,
\(\delta_i\le\|G_i-\bar G\|_{\rm op}\le\|G_i-\bar G\|_F\).
\end{proof}

\section{One-prompt deterministic map: identities and additional results}
\label{app:deterministic}

The sharp convergence region, exterior divergence alternative, and chaotic
boundary are due to \citet{liang2026gradient}; the nested-state geometry and
post-critical selection theory are due to \citet{moon2026state}.  We use
those global classifications without restating them.  The exact
transformer-to-product embedding is proved in
Appendix~\ref{app:transformer-proof}.  This section first verifies the
algebraic identities used by the stochastic analysis and then gives
additional deterministic results of this paper: transverse Lyapunov signs,
exact one-step landing sets, and rigidity of recurrent dynamics in the
strict interior.

Let \(e=ab-\mu\) and \(w=a-b\).  From
\[
a^+=a-eb,\qquad b^+=b-ea
\]
we obtain
\begin{align}
e^+
&=e^3+(\mu-2)e^2+(1-2\mu-w^2)e,\label{eq:app-e}\\
w^+&=(1+e)w.\label{eq:app-w}
\end{align}
For
\[
D=w^2-(2-\mu)(2-e),
\qquad
q_\mu(e)=e^2+\mu e+1,
\]
one has
\[
2-e^+=(2-e)q_\mu(e)+ew^2.
\]
Consequently,
\begin{align*}
D^+
&=(1+e)^2w^2-(2-\mu)(2-e^+)\\
&=\bigl((1+e)^2-(2-\mu)e\bigr)w^2
 -(2-\mu)(2-e)q_\mu(e)\\
&=q_\mu(e)D,
\end{align*}
which proves \eqref{eq:D-deterministic}.  Also,
\[
D
=a^2+b^2-\mu ab+\mu^2-4
=\frac{2+\mu}{4}
\left[
\frac{(a+b)^2}{2+\mu}
+\frac{(a-b)^2}{2-\mu}
-4
\right](2-\mu),
\]
and direct rearrangement gives the ellipse \eqref{eq:ellipse}.

\subsection{Transverse Lyapunov signs}

On the boundary \(D=0\), substitution in \eqref{eq:app-e} gives the
Chebyshev map
\begin{equation}
\label{eq:app-chebyshev-map}
    e^+=C(e):=e^3-3e,
\end{equation}
while a perturbation in the \(D\)-coordinate is multiplied by
\(q_\mu(e)\).  On the balanced line \(w=0\), the error follows
\begin{equation}
\label{eq:app-balanced-map}
    e^+=F_\mu(e):=
    e^3+(\mu-2)e^2+(1-2\mu)e,
\end{equation}
and a transverse \(w\)-perturbation is multiplied by \(1+e\).

\begin{theorem}[Transverse repulsion of the boundary ellipse]
\label{thm:det-ellipse-transverse}
Fix \(0<\mu<2\).  Let \(\nu\) be a \(C\)-invariant probability measure
on \([-2,2]\) such that \(\nu(\{-2,2\})=0\) and
\(\log(2-e)\), \(\log(2+e)\), \(\log|1+e|\), and
\(\log|1-e|\) are \(\nu\)-integrable.  Then the transverse exponent in
the \(D\)-coordinate satisfies
\[
    \int\log q_\mu(e)\,d\nu(e)\ge 0,
\]
with equality only for \(\nu=\delta_0\).  The endpoint invariant measures
\(\delta_{-2}\) and \(\delta_2\) have strictly positive exponent.
\end{theorem}

\begin{proof}
Set
\[
    A_\mu=\frac{2+\mu}{4},
    \qquad
    B_\mu=\frac{2-\mu}{4}.
\]
Both weights are positive, \(A_\mu+B_\mu=1\), and
\begin{equation}
\label{eq:app-q-convex}
    q_\mu(e)
    =A_\mu(1+e)^2+B_\mu(1-e)^2.
\end{equation}
Weighted AM--GM therefore gives
\begin{equation}
\label{eq:app-q-log-bound}
    \log q_\mu(e)
    \ge
    2A_\mu\log|1+e|+2B_\mu\log|1-e|,
\end{equation}
with equality exactly when \(e=0\).

The factorizations
\[
    2-C(e)=(2-e)(1+e)^2,
    \qquad
    2+C(e)=(2+e)(1-e)^2
\]
and \(C\)-invariance imply
\[
    \int\log|1+e|\,d\nu(e)=0,
    \qquad
    \int\log|1-e|\,d\nu(e)=0.
\]
Integrating \eqref{eq:app-q-log-bound} proves nonnegativity.  Equality
forces equality in \eqref{eq:app-q-convex} almost surely, hence
\(e=0\) almost surely and \(\nu=\delta_0\).  Finally,
\(q_\mu(2)=5+2\mu>1\) and \(q_\mu(-2)=5-2\mu>1\).
\end{proof}

Thus the chaotic boundary organizes the basin geometry but is transversely
repelling for every nontrivial invariant measure covered by the theorem.
The balanced line has the opposite sign.

\begin{lemma}
\label{lem:det-balanced-inequality}
For \(1<\mu<2\) and \(e\in[-\mu,2]\),
\[
    |1+e|^\mu\le q_\mu(e),
\]
with equality if and only if \(e=0\).
\end{lemma}

\begin{proof}
Write \(c=2-\mu\in(0,1)\).  If \(e\ge-1\), set \(t=1+e\ge0\).  The
claim becomes
\[
    h(t):=t^2-ct+c-t^{2-c}\ge0.
\]
Here \(h(1)=h'(1)=0\), and
\[
    h''(t)=2-(2-c)(1-c)t^{-c}
\]
changes sign exactly once, at a point in \((0,1)\).  Since
\(h'(0+)=-c<0\), it follows that \(h\) decreases on \((0,1)\) and
increases on \((1,\infty)\), so equality holds only at \(t=1\).
If \(e<-1\), set \(t=-(1+e)>0\).  The preceding case gives
\[
    t^{2-c}\le t^2-ct+c<t^2+ct+c=q_\mu(e),
\]
which is strict.
\end{proof}

\begin{theorem}[Transverse attraction of balanced invariant measures]
\label{thm:det-balanced-transverse}
Fix \(1<\mu<2\).  Let \(\nu\) be an \(F_\mu\)-invariant probability
measure on \([-\mu,2]\) such that \(\nu(\{2\})=0\),
\(\log(2-e)\) is \(\nu\)-integrable, and \(\log|1+e|\) is
\(\nu\)-integrable.  Then
\[
    \int\log|1+e|\,d\nu(e)\le0,
\]
with equality only for \(\nu=\delta_0\).
\end{theorem}

\begin{proof}
Direct expansion gives
\[
    2-F_\mu(e)=(2-e)q_\mu(e).
\]
Invariance and the integrability assumptions therefore imply
\(\int\log q_\mu(e)\,d\nu(e)=0\).  By
Lemma~\ref{lem:det-balanced-inequality},
\[
    \mu\int\log|1+e|\,d\nu(e)
    \le\int\log q_\mu(e)\,d\nu(e)=0.
\]
Equality forces \(e=0\) almost surely, and hence \(\nu=\delta_0\).
\end{proof}

\subsection{Landing sets and interior recurrence rigidity}

The exact coordinate formulas also identify all points that land on three
special invariant sets in one step.

\begin{proposition}[Exact one-step landing sets]
\label{prop:det-landing-sets}
Let \(s=a+b\).  Under \(\Phi_\mu\):
\begin{enumerate}[leftmargin=*,nosep]
    \item \(w^+=0\) if and only if \(w=0\) or \(e=-1\);
    \item \(s^+=0\) if and only if \(s=0\) or \(e=1\);
    \item \(e^+=0\) if and only if \(e=0\) or \(D=e^2-3\).
\end{enumerate}
\end{proposition}

\begin{proof}
The update gives
\[
    w^+=(1+e)w,
    \qquad
    s^+=(1-e)s.
\]
Moreover, eliminating \(w^2\) from \eqref{eq:app-e} using the definition
of \(D\) yields
\[
    e^+=e(e^2-3-D).
\]
The three equivalences follow.
\end{proof}

For the recurrence statement, define
\[
    u=s^2,\qquad v=w^2,\qquad
    R=\frac{u}{2+\mu}+\frac{v}{2-\mu},
\]
and let
\[
    \mathcal I_\mu
    =
    \{(a,b):u>0,\ v>0,\ R<4\}.
\]
This is the genuinely two-dimensional part of the ellipse interior,
excluding the balanced and anti-balanced axes.

\begin{theorem}[No recurrent off-balanced dynamics in the strict interior]
\label{thm:det-interior-rigidity}
Fix \(0<\mu<2\).  Let \(\nu\) be a
\(\Phi_\mu\)-invariant probability measure such that
\(\nu(\mathcal I_\mu)=1\) and
\(\log u\), \(\log v\), and \(\log(4-R)\) are
\(\nu\)-integrable.  Then \(\nu(\{e=0\})=1\).  Consequently,
\(\mathcal I_\mu\) contains no periodic orbit of period greater than one
and no invariant measure satisfying these hypotheses with genuinely
two-dimensional nonzero-error recurrence.
\end{theorem}

\begin{proof}
The light-cone coordinates and the boundary identity give
\[
    u^+=(1-e)^2u,\qquad
    v^+=(1+e)^2v.
\]
Since
\[
    D=\frac{4-\mu^2}{4}(R-4),
\]
the identity \(D^+=q_\mu(e)D\) also gives
\[
    4-R^+=q_\mu(e)(4-R).
\]
Invariance and the three integrability assumptions imply
\begin{equation}
\label{eq:app-rigidity-zero-integrals}
    \int\log|1-e|\,d\nu
    =
    \int\log|1+e|\,d\nu
    =
    \int\log q_\mu(e)\,d\nu
    =0.
\end{equation}
Integrating the weighted AM--GM inequality
\eqref{eq:app-q-log-bound} and using
\eqref{eq:app-rigidity-zero-integrals} shows that equality holds almost
surely in \eqref{eq:app-q-convex}.  Thus \(e=0\) almost surely.

For a periodic orbit in \(\mathcal I_\mu\), the uniform measure on the
orbit satisfies the hypotheses above.  Hence every point on the orbit has
\(e=0\), where \(\Phi_\mu\) is the identity, so the orbit is fixed.
\end{proof}

\begin{corollary}[Finite landing on the post-critical zero-error curve]
\label{cor:det-interior-large-step}
Assume \(1<\mu<2\).  The zero-error fixed curve
\(\M_\mu=\{(a,b):ab=\mu\}\) is normally repelling.  Any orbit satisfying
\(e_k\to0\) lands on \(\M_\mu\) in finite time.  Therefore the set of
initial conditions in \(\mathcal I_\mu\) that converge to \(\M_\mu\) is
contained in
\[
    \bigcup_{j\ge0}\Phi_\mu^{-j}(\M_\mu),
\]
a set of two-dimensional Lebesgue measure zero.
\end{corollary}

\begin{proof}
Equation~\eqref{eq:app-e} can be written
\[
    e^+=e\,M(e,w),
    \qquad
    M(e,w)=e^2+(\mu-2)e+1-2\mu-w^2.
\]
At \(e=0\), \(M(0,w)=1-2\mu-w^2<-1\), proving normal repulsion.  Choose
\(\delta>0\) small enough that
\[
    \delta^2+(2-\mu)\delta+1-2\mu<-1.
\]
For \(0<|e|\le\delta\),
\[
    M(e,w)
    \le e^2+(2-\mu)|e|+1-2\mu<-1,
\]
so \(|e^+|>|e|\).  An orbit with \(e_k\to0\) therefore cannot remain
off \(e=0\) for all sufficiently large \(k\); it must land there in finite
time.

For each \(j\), the set \(\Phi_\mu^{-j}(\M_\mu)\) is the zero set of the
nonzero polynomial \(e(\Phi_\mu^j(a,b))\).  It has planar Lebesgue measure
zero, and so does the countable union over \(j\).
\end{proof}

These results complement the global basin descriptions of
\citet{liang2026gradient} and \citet{moon2026state}: the ellipse is a
repelling separatrix, balanced invariant dynamics have the opposite
transverse sign, and the strict interior contains no hidden recurrent
off-balanced component under the stated integrability conditions.

\section{Proofs for unequal-curvature mini-batches}
\label{app:minibatch-proofs}

\begin{proof}[Proof of Lemma~\ref{thm:moving-geometry}]
For a batch \(\mathsf B\), the gradient of the average of
\eqref{eq:sample-loss} is
\[
\partial_a\ell_{\mathsf B}
=(c_{\mathsf B}ab-h_{\mathsf B})b,
\qquad
\partial_b\ell_{\mathsf B}
=(c_{\mathsf B}ab-h_{\mathsf B})a,
\]
where
\[
c_{\mathsf B}=m^{-1}\sum_{i\in\mathsf B}\gamma_i^2,
\qquad
h_{\mathsf B}=m^{-1}\sum_{i\in\mathsf B}\gamma_i y_i.
\]
The fixed normalization in Section~\ref{sec:reduction} gives
\eqref{eq:Psi}.

Let \((\widetilde A,\widetilde B)=\sqrt\alpha(A,B)\).  Multiplying the
batch update by \(\sqrt\alpha\) yields
\[
\widetilde A^+
=\widetilde A-(\widetilde A\widetilde B-\nu)\widetilde B,
\qquad
\widetilde B^+
=\widetilde B-(\widetilde A\widetilde B-\nu)\widetilde A.
\]
Thus \(T_\alpha\circ\Psi_{\alpha,\nu}
=\Phi_\nu\circ T_\alpha\).  Substituting
\(\widetilde A\pm\widetilde B=\sqrt\alpha(A\pm B)\) in
\eqref{eq:ellipse} gives the displayed batch ellipse.

Finally, for \(p=AB\),
\[
\alpha_{\mathsf B}p-\nu_{\mathsf B}
=p-\mu-\left[
\nu_{\mathsf B}-\mu-(\alpha_{\mathsf B}-1)p
\right]
=p-\mu-\rho_{\mathsf B}(A,B).
\]
This proves both equalities in \eqref{eq:Psi-rho}.  Unbiased sampling gives
\(\E[\alpha_{\mathsf B}]=1\) and
\(\E[\nu_{\mathsf B}]=\mu\), hence conditional centering of \(\rho_{\mathsf
B}\).
\end{proof}

\section{Proofs of the crossing results}
\label{app:crossing-proofs}

\begin{proof}[Proof of Proposition~\ref{thm:safe-interval}]
By \eqref{eq:Psi-rho}, the batch update is
\[
A^+=A-(e-\rho)B,
\qquad
B^+=B-(e-\rho)A.
\]
Hence
\[
w^+=(1+e-\rho)w
\]
and
\[
e^+
=e-(e-\rho)\bigl(w^2+2\mu+2e\bigr)
 +(e-\rho)^2(\mu+e).
\]
Substituting these expressions into
\[
D_\mu(A^+,B^+)=(w^+)^2-(2-\mu)(2-e^+)
\]
and collecting powers of \(\rho\) gives
\[
D_\mu(A^+,B^+)
=q_\mu(e)D
-\rho\left[
(2e+\mu-\rho)D+(4-\mu^2)(e-\rho)
\right].
\]
Equivalently,
\[
D_\mu(A^+,B^+)
=H_\mu\rho^2-J_\mu\rho+q_\mu(e)D,
\]
which is \eqref{eq:D-rho}.

For \(0<\mu<2\),
\[
H_\mu(A,B)
=\frac{2-\mu}{4}(A+B)^2
+\frac{2+\mu}{4}(A-B)^2>0
\]
away from the origin, while \(q_\mu(e)>0\).  If \(D<0\), the quadratic in
\(\rho\) has positive leading coefficient and a negative value at zero.  Its
two roots are therefore real and have opposite signs.  The quadratic is
negative exactly between the roots, proving \eqref{eq:rho-pm} and the stated
criterion.

When \(D=0\), \(H_\mu=4-\mu^2\) and
\(J_\mu=(4-\mu^2)e\), so
\[
D_\mu(A^+,B^+)=(4-\mu^2)\rho(\rho-e).
\]
\end{proof}

\begin{proof}[Proof of Corollary~\ref{cor:batch-bound}]
For \(p=AB\), the definitions give
\[
\rho_{\mathsf B}(A,B)
=\frac1m\sum_{i\in\mathsf B}
\left[
\eta(\gamma_i y_i-\bar h)
-p\left(\frac{\gamma_i^2}{\bar c}-1\right)
\right]
=\frac1m\sum_{i\in\mathsf B}Z_i(p).
\]
If \(\kappa(p)=0\), centered sub-Gaussianity forces \(Z_i(p)=0\) almost
surely, so \(\rho_{\mathsf B}=0\) and an interior state cannot cross.
Assume henceforth that \(\kappa(p)>0\).
The average of \(m\) independent
\(\kappa(p)^2\)-sub-Gaussian variables is
\(\kappa(p)^2/m\)-sub-Gaussian.  By
Proposition~\ref{thm:safe-interval}, crossing requires either
\(\rho_{\mathsf B}\le\rho_-\) or
\(\rho_{\mathsf B}\ge\rho_+\).  Applying the two one-sided tail bounds and
a union bound gives
\[
\Pr(D_\mu(A^+,B^+)\ge0\mid A,B)
\le
\exp\left(-\frac{m\rho_-^2}{2\kappa(p)^2}\right)
+
\exp\left(-\frac{m\rho_+^2}{2\kappa(p)^2}\right),
\]
which implies \eqref{eq:crossing-bound}.
\end{proof}

\paragraph{Finite-population scale.} Centering does not change the range of a finite random variable.  The range
of \(Z_i(p)\) is at most
\[
\eta\left(\max_i\gamma_i y_i-\min_i\gamma_i y_i\right)
+|p|\left(
\max_i\frac{\gamma_i^2}{\bar c}
-\min_i\frac{\gamma_i^2}{\bar c}
\right)
=2(\kappa_0+|p|\kappa_1).
\]
Hoeffding's lemma therefore makes \(Z_i(p)\)
\((\kappa_0+|p|\kappa_1)^2\)-sub-Gaussian.

\begin{proof}[Proof of Theorem~\ref{thm:sequential-crossing}]
For \(k<T\), Corollary~\ref{cor:batch-bound} and the fresh-batch assumption
give
\[
\Pr(T=k+1\mid\Fcal_k)
\le
2\exp\left(
-\frac{m}{2}
\frac{\tau_\mu(A_k,B_k)^2}{\kappa(A_kB_k)^2}
\right).
\]
On \(\mathcal G_{K,\chi}\), this conditional probability is at most
\(2e^{-m\chi^2/2}\) at every pre-crossing step.  Decomposing
\(\{T\le K\}\) according to the first crossing time and using the tower
property,
\begin{align*}
\Pr(T\le K,\mathcal G_{K,\chi})
&\le
\sum_{k=0}^{K-1}
\E\left[
\mathbf 1_{\{T>k,\,
\tau_\mu(A_k,B_k)/\kappa(A_kB_k)\ge\chi\}}
\Pr(T=k+1\mid\Fcal_k)
\right]\\
&\le 2K e^{-m\chi^2/2}.
\end{align*}
\end{proof}

\paragraph{Proof of \eqref{eq:solution-crossing}.}
At \(\theta_\mu=(\sqrt\mu,\sqrt\mu)\), \(e=0\), and
\[
A^+=B^+=\sqrt\mu(1+\rho).
\]
Thus
\[
e^+=\mu(\rho^2+2\rho),
\qquad
D_\mu(A^+,B^+)=-(2-\mu)(2-e^+).
\]
Since \(2-\mu>0\), the image is outside iff \(e^+>2\), proving
\eqref{eq:solution-crossing}.

\section{Proofs of stochastic drift, escape, and nonconvergence}
\label{app:moment-proofs}

\begin{proof}[Proof of Proposition~\ref{thm:moments}]
Suppress the time index.  Equation~\eqref{eq:Psi-rho} gives
\[
(A^+,B^+)
=\Phi_\mu(A,B)+\rho(B,A).
\]
Conditional centering proves \eqref{eq:mean-map}, and the covariance of the
rank-one perturbation is \eqref{eq:cov-map}.

The imbalance identity is
\[
w^+=(1+e-\rho)w.
\]
Taking the conditional second moment gives \eqref{eq:w2}.  From
\eqref{eq:D-rho},
\[
\E[D^+\mid\Fcal]
=H_\mu v+q_\mu(e)D.
\]
Using \(H_\mu=D+4-\mu^2\) yields \eqref{eq:D-drift}.
\end{proof}

\begin{proof}[Proof of Proposition~\ref{prop:frozen-vs-moment}]
At \(\theta_\mu\), one has \(AB=\mu\) and
\(\rho=\nu-\alpha\mu\).  Equation~\eqref{eq:Psi-rho} gives
\[
\theta_\mu^+
=\theta_\mu+\rho(\sqrt\mu,\sqrt\mu)
=(1+\rho)\theta_\mu.
\]
The exact imbalance identity is
\[
w^+=(1+\alpha AB-\nu)w,
\]
so its transverse linearization at \(\theta_\mu\) is
\(\delta w^+=(1-\rho)\delta w\).  Taking second moments and using
\(\E\rho=0\) gives \eqref{eq:local-dichotomy}.

If \(\rho<1\) almost surely, strict concavity of \(\log\) gives
\[
\E\log(1-\rho)
\le \log\E(1-\rho)=0,
\]
with strict inequality for nonconstant \(\rho\).  Moreover,
\(\E(1-\rho)^2=1+v\), proving the RMS formula.
If \(|\rho|\le r<1\), the power series
\[
\log(1-\rho)=-\rho-\frac{\rho^2}{2}
-\sum_{j=3}^\infty\frac{\rho^j}{j}
\]
and
\[
\left|\sum_{j=3}^\infty\frac{\rho^j}{j}\right|
\le \frac{|\rho|^3}{3(1-r)}
\]
give the stated remainder bound after expectation.
\end{proof}

\paragraph{Derivation of the random cocycle.}
Subtracting the two coordinates of \eqref{eq:Psi} gives
\[
A^+-B^+
=(1+\alpha AB-\nu)(A-B),
\]
which is \eqref{eq:random-cocycle}.  On \(A=B=r\), the scalar update is
\[
r^+=r(1+\nu-\alpha r^2).
\]
Linearizing the exact imbalance identity and iterating gives the product of
random transverse multipliers stated in the main text.

\begin{proof}[Proof of Theorem~\ref{thm:as-escape}]
Every \(g_j\) is a nonconstant polynomial.  Hence every finite composition
is nonconstant with finitely many zeros, and the union of these zero sets
over the countable collection of finite words is the countable set
\(\mathcal L\).

We make the recurrence step explicit.  Let
\(\mathcal F_k=\sigma(\xi_0,\ldots,\xi_{k-1})\), fix a rational open
interval \(U\), and fix a word
\(\boldsymbol j=(j_0,\ldots,j_{\ell-1})\).  On the event that \(U\) is
visited infinitely often, define
\[
\sigma_1=\inf\{k:x_k\in U\},\qquad
\sigma_{i+1}=\inf\{k\ge\sigma_i+\ell:x_k\in U\}.
\]
Let \(E_i\) be the event that the block beginning at \(\sigma_i\) equals
\(\boldsymbol j\).  Then
\[
E_i\in\mathcal F_{\sigma_i+\ell}
\subseteq\mathcal F_{\sigma_{i+1}},
\qquad
\Pr(E_i\mid\mathcal F_{\sigma_i})
=p_{\boldsymbol j}>0.
\]
L\'evy's extension of Borel--Cantelli therefore gives infinitely many
\(E_i\).  Taking the countable intersection over rational \(U\) and finite
words, then using continuity, shows that the omega-limit set \(\Omega\) of
any bounded path satisfies
\begin{equation}
\label{eq:app-word-invariance}
    g_{\boldsymbol j}(\Omega)\subseteq\Omega
\end{equation}
for every finite word \(\boldsymbol j\).

Now suppose \(x_0\notin\mathcal L\) and a sample path is bounded.  Then
\(x_k>0\) for all \(k\).  With
\(\nu_{\min}=\min_j\nu_j>0\) and
\(\alpha_{\max}=\max_j\alpha_j\), every
\(0<x<\nu_{\min}/(2\alpha_{\max})\) satisfies
\[
    1+\nu_j-\alpha_jx\ge1+\frac{\nu_{\min}}2>1
\]
for every \(j\).  Thus a positive path cannot converge to zero, so its
compact omega-limit set is not \(\{0\}\).  Equation
\eqref{eq:app-word-invariance} makes \(\Omega\) a nontrivial compact common
forward-invariant set, contradicting the hypothesis.

The path is therefore unbounded almost surely.  Let
\(\alpha_{\min}=\min_j\alpha_j>0\) and
\(\nu_{\max}=\max_j\nu_j\).  For sufficiently large \(R\),
\[
    \alpha_{\min}R-1-\nu_{\max}>1.
\]
Once \(x_k>R\), every possible update multiplies \(x_k\) by a factor bounded
below by a fixed constant greater than one.  Hence \(x_k\to\infty\).
Conversely, if a nontrivial compact common forward-invariant \(K\) exists,
every switching trajectory initialized in \(K\) remains in \(K\).
\end{proof}

\begin{proof}[Proof of Corollary~\ref{cor:two-type-as-escape}]
A batch containing \(j\) high-type prompts has the displayed
\((\alpha_j,\nu_j)\), and the balanced line \(A=B=r\) gives
\[
    x^+=g_j(x),\qquad x=r^2.
\]
Let \(K\subset[0,\infty)\) be nonempty, compact, and common
forward-invariant.  If \(K\ne\{0\}\), set \(M=\max K>0\).  For the all-low
and all-high maps,
\[
L(x)=x\left(\frac{11}{10}-\frac{x}{2}\right)^2,\qquad
H(x)=x\left(\frac{18}{5}-3x\right)^2.
\]
Since \(H(M)\le M\),
\[
    \frac{13}{15}\le M\le\frac{23}{15}.
\]
On this interval \(L\) is decreasing, so
\[
    \frac{23}{135}\le L(M)\le\frac{52}{135}.
\]
The map \(H\) is increasing on this latter interval, and
\[
H(L(M))
\ge H\left(\frac{23}{135}\right)
=\frac{444383}{273375}>\frac85.
\]
For \(z>8/5\),
\[
H(z)>\frac85\left(\frac65\right)^2
=\frac{288}{125}>\frac{23}{15}.
\]
Thus \(H(H(L(M)))>M\), although common forward invariance requires the word
\(H\circ H\circ L\) to map \(M\) back into \(K\), a contradiction.
Theorem~\ref{thm:as-escape} now gives almost-sure divergence outside the
landing set.

Finally suppose \(5\nmid m\), let \(x_0=3/5\), and write \(v_5\) for the
\(5\)-adic valuation.  The batch multiplier is
\[
1+\nu_j-\alpha_jx
=\frac{11m+25j-5(m+5j)x}{10m}.
\]
At \(x_0=3/5\), its numerator is \(8m+10j\not\equiv0\pmod5\), so
\(v_5(x_1)=-3\).  If \(v_5(x_k)=-q\) with \(q\ge3\), the two numerator
terms have valuations \(1-q<0\) and \(0\), respectively.  Hence the
multiplier has valuation \(-q\) and
\(v_5(x_{k+1})=-3q\).  Inductively \(v_5(x_k)=-3^k\), so the plotted
initialization never lands at zero.
\end{proof}

\begin{proof}[Proof of Proposition~\ref{thm:no-convergence}]
Independence and positive sampling mass imply, by Borel--Cantelli, that each
element of \(\mathfrak B\) is selected along an infinite index set.  Assume
\(z_k=(A_k,B_k)\to z_\infty\ne0\); convergence also forces the increments
\(z_{k+1}-z_k\) to vanish.  Choose any \(\mathsf B\in\mathfrak B\) and let
\(k_j\) enumerate its occurrence times.  Taking \(j\to\infty\) in the update
at \(k_j\) gives
\[
0
=-(\alpha_{\mathsf B}A_\infty B_\infty-\nu_{\mathsf B})
  (B_\infty,A_\infty).
\]
The second factor is nonzero, so
\[
\alpha_{\mathsf B}A_\infty B_\infty=\nu_{\mathsf B}.
\]
This holds for every recurring batch and proves
\eqref{eq:common-interpolation}.  If all ratios equal \(p_\star\), then
\(\alpha_{\mathsf B}AB-\nu_{\mathsf B}=0\) on \(AB=p_\star\) for every
batch, proving sharpness.
\end{proof}

\section{Finite-token softmax reduction and boundary transfer}
\label{app:softmax}

\subsection{Gaussian population ICL identity}

Let \(\widetilde z=(x,y)\), where \(x\sim N(0,I_d)\) and
\(y=w_\star^\top x\).  Its covariance is
\[
\Sigma_z=
\begin{pmatrix}
I_d&w_\star\\
w_\star^\top&\|w_\star\|^2
\end{pmatrix}.
\]
For a fixed query token \(\widetilde z_q=(x_q,0)\), let
\(t=Q^\top\widetilde z_q\).  The Gaussian moment-generating function is
\[
\E[\exp(t^\top\widetilde z')]
=\exp\left(\frac12t^\top\Sigma_z t\right).
\]
Differentiating in \(t\) gives
\[
\E[\exp(t^\top\widetilde z')\widetilde z']
=\Sigma_z t\exp\left(\frac12t^\top\Sigma_z t\right).
\]
Taking the ratio proves \eqref{eq:gaussian-softmax}.

Write
\[
\ell=\|w_\star\|>0,\qquad h=1+\ell^2,\qquad
c=\|x_q\|,\qquad \sigma=\sgn(y_q),\qquad \bar y=|y_q|,
\]
and define
\[
r_\star=\frac{(w_\star,\ell^2)}{\ell h},\qquad
v=\frac{\widetilde z_q}{c}.
\]
Then
\[
\Sigma_zr_\star=hr_\star,\qquad
\|r_\star\|^2=h^{-1},\qquad
r_\star^\top\widetilde z=\frac{y}{\ell}=:Z\sim N(0,1).
\]
Consider the rank-one sector
\[
\mathcal A=\sigma\alpha r_\star^\top,\qquad
Q=\beta vr_\star^\top.
\]
The score is \(sZ\), where \(s=c\beta\).  By
\eqref{eq:gaussian-softmax}, the tilted population mean is
\(sh r_\star\), and hence the population prediction is
\[
\sigma\alpha r_\star^\top(sh r_\star)
=\sigma c\alpha\beta.
\]
Thus the loss is \(\frac12(c\alpha\beta-\bar y)^2\), and its ambient gradients
are tangent to the rank-one sector.  Although \(r_\star\) uses both token
blocks, on the noiseless regression support it reads exactly the normalized
prompt label \(y/\ell\).

Let \(\eta_{\rm amb}\) be the ambient parameter learning rate.  Orthogonal
projection onto the non-unit directions \(r_\star\) and
\(vr_\star^\top\) multiplies the scalar coefficient step by
\(\|r_\star\|^{-2}=h\), so define
\(\eta_{\rm sm}=h\eta_{\rm amb}\).  Under
\[
A=\sqrt{\eta_{\rm sm}}\,c\alpha,\qquad
B=\sqrt{\eta_{\rm sm}}\,c\beta,\qquad
\mu=\eta_{\rm sm}c\bar y,
\]
population gradient descent is exactly \(\Phi_\mu\).

\subsection{The exact projected finite-token map}

Let \(\widetilde z_j=(x_j,y_j)\) be independent regression tokens and set
\[
Z_j=r_\star^\top\widetilde z_j=\frac{y_j}{\ell}.
\]
Define
\[
\pi_j(s)=\frac{\exp(sZ_j)}{\sum_{\ell=1}^N\exp(sZ_\ell)},\quad
\widehat M_N(s)=\sum_j\pi_j(s)\widetilde z_j,\quad
\widehat C_N(s)=
\sum_j\pi_j(s)(\widetilde z_j-\widehat M_N)
(\widetilde z_j-\widehat M_N)^\top.
\]
The active scalar moments are
\[
\widehat m_N(s)=r_\star^\top\widehat M_N(s),\qquad
\widehat v_N(s)=r_\star^\top\widehat C_N(s)r_\star
=\frac{d}{ds}\widehat m_N(s).
\]
On the rank-one sector the finite-token prediction is
\(\sigma\alpha\widehat m_N(s)\), which explicitly uses the labels \(y_j\),
and the sign-corrected residual is
\[
\widehat R_N=\alpha\widehat m_N(s)-\bar y.
\]
The restricted derivatives are
\[
\partial_\alpha\widehat{\mathcal L}_N
=\widehat R_N\widehat m_N(s),\qquad
\partial_\beta\widehat{\mathcal L}_N
=\widehat R_N\alpha c\widehat v_N(s).
\]
Projection of an ambient step of size \(\eta_{\rm amb}\) is therefore the
same scalar coefficient step with \(\eta_{\rm sm}=h\eta_{\rm amb}\).
Substituting the normalization above, for which
\(s=B/\sqrt{\eta_{\rm sm}}\), yields exactly
\eqref{eq:finite-softmax-map}.  This also shows why finite-token softmax
retains \(\eta_{\rm sm}\) after the population dynamics have collapsed to
the one-parameter family \(\Phi_\mu\).
Any conventional temperature factor can be absorbed into \(s\).  Including
the query token itself in the attention denominator changes the empirical
measure by one atom of mass \(1/(N+1)\); on compact score sets this is an
\(O(N^{-1})\) correction and does not alter the limits below.

\subsection{Explicit Gaussian fluctuations}

\begin{proof}[Proof of Proposition~\ref{prop:softmax-clt}]
Let \(Z\sim N(0,1)\), \(W=\exp(sZ)\), and
\(a_s=\E W=\exp(s^2/2)\).  The empirical tilted mean and variance are smooth
functions of the sample averages of \(W,WZ,WZ^2\).  Their influence
functions at the population moments are
\[
\psi_m(Z)=\frac{W}{a_s}(Z-s),\qquad
\psi_v(Z)=\frac{W}{a_s}\bigl((Z-s)^2-1\bigr).
\]
The multivariate central limit theorem and delta method therefore apply.
For every integrable \(f\),
\[
\frac{\E[W^2f(Z-s)]}{a_s^2}
=e^{s^2}\E[f(Y)],
\qquad Y\sim N(s,1).
\]
Using
\[
\E Y^2=1+s^2,\quad
\E[Y(Y^2-1)]=s^3+2s,\quad
\E[(Y^2-1)^2]=s^4+4s^2+2
\]
gives the displayed covariance \(\Sigma_{\rm sm}(s)\).  Since lognormal
weights have moments of every order, a standard denominator truncation
gives uniform integrability, so the corresponding second moments converge as
well.

View the two coordinates of \eqref{eq:finite-softmax-map} as a smooth
function of \((m,v)\).  At the population value
\((m,v)=(s,1)\), its Jacobian is
\[
\begin{pmatrix}
-\sqrt{\eta_{\rm sm}}(2AB-\mu)&0\\
-\sqrt{\eta_{\rm sm}}A^2&-A(AB-\mu)
\end{pmatrix}
=J_\eta(A,B).
\]
Another application of the delta method proves the map limit.

For the final statement, let
\[
g_\star
=\nabla D_\mu(\Phi_\mu(\zeta_\star)),\qquad
\sigma_D^2
=g_\star^\top
J_\eta(\zeta_\star)\Sigma_{\rm sm}(s_\star)
J_\eta(\zeta_\star)^\top g_\star,
\quad
s_\star=\frac{B_\star}{\sqrt{\eta_{\rm sm}}}.
\]
Taylor expansion of \(D_\mu\) at the population image and Slutsky's theorem
give
\[
\sqrt N\,D_\mu(
\widehat\Phi_{\mu,N}^{(\eta_{\rm sm})}(\zeta_N))
\Rightarrow \mathcal N(-\delta,\sigma_D^2).
\]
Taking the probability of the nonnegative half-line proves the crossing
limit.
\end{proof}

\begin{corollary}[Batch--context tradeoff]
\label{cor:softmax-batch-context}
Let \(m\) independent prompts share the same population scalar sector, and
average their projected finite-token gradients in one update.  If
\(\widehat\Phi_{\mu,m,N}^{(\eta_{\rm sm})}\) denotes that batch map, then
for fixed \(m\) as \(N\to\infty\),
\[
\sqrt{mN}\left(
\widehat\Phi_{\mu,m,N}^{(\eta_{\rm sm})}(\zeta)
-\Phi_\mu(\zeta)\right)
\Rightarrow
\mathcal N(0,J_\eta\Sigma_{\rm sm}J_\eta^\top).
\]
More generally, let \(P\) denote a random prompt type, let
\(\widehat F_N(P)\) be its finite-token update, and let \(F_\infty(P)\) be
its population update.  Assume
\[
\E[\widehat F_N(P)\mid P]\to F_\infty(P)
\quad\text{in }L^2(P),
\]
and
\[
N\operatorname{Cov}(\widehat F_N(P)\mid P)
\to V(P)
\quad\text{in }L^1(P).
\]
Then the exact conditional variance decomposition is
\[
\operatorname{Cov}\!\left(\frac1m\sum_{i=1}^m\widehat F_N(P_i)\right)
=
\frac1m\operatorname{Cov}_P\!\left(\E[
\widehat F_N(P)\mid P]\right)
+\frac1m\E_P\operatorname{Cov}\!\left(
\widehat F_N(P)\mid P\right).
\]
The first term converges to the between-prompt switching covariance divided
by \(m\); after multiplication by \(N\), the second converges to
\(\E_P V(P)/m\).  Uniformly bounded scores and an integrable envelope for
the relevant parameter and token moments are sufficient for these
assumptions, with
\(V(P)=J_P\Sigma_{\rm sm}(s_P)J_P^\top\).
\end{corollary}

\begin{proof}
In the homogeneous case, average the \(m\) independent influence-function
expansions from Proposition~\ref{prop:softmax-clt}.  The general identity is
the law of total covariance.  The stated \(L^2\) and \(L^1\) assumptions
justify passage of the prompt expectation through the conditional limits.
\end{proof}

The common factor \(e^{s^2}\) has a direct importance-sampling
interpretation:
\[
\frac{(\E W)^2}{\E W^2}=e^{-s^2}.
\]
Thus \(N_{\rm eff}=Ne^{-s^2}\) is the population effective sample size of
the softmax weights.  The polynomial factors in
\(\Sigma_{\rm sm}(s)\) distinguish mean, curvature, and map fluctuations,
but all of them share the same exponential loss of effective context.

\subsection{Ambient leakage from the rank-one sector}

\begin{proposition}[Finite-token softmax ICL leakage]
\label{prop:softmax-leakage}
Let
\[
P_\star=hr_\star r_\star^\top
\]
be the orthogonal projector onto the active token direction, and let
\(P_\perp\) project onto
\(\{(u,0):u\perp w_\star\}\), the remaining support directions of
\(\Sigma_z\).  At fixed score \(s\),
\begin{align}
\sqrt N(I-P_\star)\widehat M_N(s)
&\Rightarrow
\mathcal N\!\left(0,e^{s^2}P_\perp\right),\label{eq:softmax-mean-leakage}\\
\sqrt N(I-P_\star)\widehat C_N(s)r_\star
&\Rightarrow
\mathcal N\!\left(0,e^{s^2}(1+s^2)P_\perp\right).
\label{eq:softmax-cov-leakage}
\end{align}
Moreover, if \(\Pi\) projects the paired parameter gradient onto the
rank-one \((\alpha,\beta)\) sector, then
\begin{equation}
\label{eq:softmax-gradient-leakage}
\|(I-\Pi)\nabla\widehat{\mathcal L}_N\|^2
=
\widehat R_N^2
\left(
\|(I-P_\star)\widehat M_N\|^2
+\alpha^2c^2\|(I-P_\star)\widehat C_Nr_\star\|^2
\right).
\end{equation}
Consequently the ambient leakage is
\(O_{\Pr}(\sqrt{(d-1)e^{s^2}/N})\) on bounded parameter sets.
\end{proposition}

\begin{proof}
Write \(\widehat w=w_\star/\ell\).  Every regression token decomposes as
\[
\widetilde z=hr_\star Z+U,\qquad
U=(x-\widehat wZ,0),\qquad
U\sim N(0,P_\perp),
\]
where \(Z\sim N(0,1)\) and \(Z,U\) are independent.  The influence
functions of the two off-sector statistics are
\[
\frac{W}{a_s}U,\qquad
\frac{W}{a_s}U(Z-s).
\]
Exponential tilting and independence give the covariances in
\eqref{eq:softmax-mean-leakage}--\eqref{eq:softmax-cov-leakage}.

For completeness, direct differentiation of the empirical attention output
gives
\[
\nabla_{\mathcal A}\widehat{\mathcal L}_N
=\sigma\widehat R_N\widehat M_N^\top,\qquad
\nabla_Q\widehat{\mathcal L}_N
=\widehat R_N\alpha c\,v r_\star^\top\widehat C_N.
\]
Projecting the first row onto \(r_\star^\top\) and the second onto
\(vr_\star^\top\) leaves the two orthogonal terms in
\eqref{eq:softmax-gradient-leakage}.
\end{proof}

\begin{lemma}[Paired boundary-scale normal perturbations]
\label{lem:softmax-paired-perturbation}
Let \(\alpha,\beta\) range over a fixed compact set, and let \(u\) be a unit
vector in \(\operatorname{range}(P_\perp)\).  Uniformly on that set, for a
deterministic sequence \(\xi_N\to0\), compare
\[
\mathcal A_{0}=\sigma\alpha r_\star^\top,\quad
Q_{0}=\beta vr_\star^\top
\]
with the paired perturbation
\[
\mathcal A_{\xi_N}
=\sigma(\alpha r_\star^\top+\xi_Nu^\top),\quad
Q_{\xi_N}
=\beta vr_\star^\top+\xi_Nvu^\top.
\]
Apply one full ambient empirical gradient step to both initializations using
the same \(N\)-token prompt, and let
\(\widehat\zeta^+_{N,0}\) and
\(\widehat\zeta^+_{N,\xi_N}\) be the normalized coordinates obtained by
projecting the two endpoints onto the rank-one sector.  Then
\[
\left\|
\widehat\zeta^+_{N,\xi_N}
-\widehat\zeta^+_{N,0}
\right\|
=O_{\Pr}\!\left(
|\xi_N|N^{-1/2}+\xi_N^2
\right).
\]
In particular, if \(\xi_N=O(N^{-1/2})\), the paired endpoint difference is
\(O_{\Pr}(N^{-1})\).
\end{lemma}

\begin{proof}
Write \(Z_j=r_\star^\top\widetilde z_j\) and
\(U_j=u^\top\widetilde z_j\).  Proposition~\ref{prop:softmax-leakage}
implies that \(Z_j,U_j\) are independent standard Gaussians.  For
\(t=c\xi_N\), define the empirical tilted average
\[
\langle f\rangle_t
=
\frac{\sum_{j=1}^N e^{sZ_j+tU_j}f(Z_j,U_j)}
{\sum_{j=1}^N e^{sZ_j+tU_j}},
\qquad s=c\beta,
\]
and let \(m_Z(t),m_U(t)\) and
\(v_{ZZ}(t),v_{UZ}(t)\) denote the corresponding means and covariances.
At \(t=0\), independence and the central limit theorem give
\[
m_U(0),\quad v_{ZU}(0),\quad
\left.\frac{d}{dt}v_{ZZ}(t)\right|_{t=0}
=O_{\Pr}(N^{-1/2}).
\]
The last derivative is an empirically tilted centered moment containing one
factor of \(U\), whose population value is zero.

All empirical derivatives through second order are ratios of averages of
exponentially weighted Gaussian polynomials.  Their denominators converge to
a positive constant and their numerators are \(O_{\Pr}(1)\).  Taylor
expansion, using \(m_Z'(0)=v_{ZU}(0)\), therefore gives
\begin{align*}
m_Z(t)-m_Z(0)
&=O_{\Pr}(|\xi_N|N^{-1/2}+\xi_N^2),\\
\xi_Nm_U(t)
&=O_{\Pr}(|\xi_N|N^{-1/2}+\xi_N^2),\\
v_{ZZ}(t)-v_{ZZ}(0)
&=O_{\Pr}(|\xi_N|N^{-1/2}+\xi_N^2),\\
\xi_Nv_{UZ}(t)
&=O_{\Pr}(|\xi_N|N^{-1/2}+\xi_N^2).
\end{align*}

The sign-corrected empirical residual is
\[
\widehat R_{\xi_N}
=\alpha m_Z(t)+\xi_Nm_U(t)-\bar y.
\]
Projection of the full ambient update gives the exact coefficient formulas
\[
\widehat\alpha_{\xi_N}^+
=\alpha-\eta_{\rm sm}\widehat R_{\xi_N}m_Z(t),
\]
\[
\widehat\beta_{\xi_N}^+
=\beta-\eta_{\rm sm}c\widehat R_{\xi_N}
\bigl(\alpha v_{ZZ}(t)+\xi_Nv_{UZ}(t)\bigr).
\]
The four estimates above imply the claimed rate for both coefficients.
Multiplication by the fixed normalization
\(\sqrt{\eta_{\rm sm}}c\) proves the result for
\(\widehat\zeta^+_{N,\xi_N}\).
\end{proof}

\subsection{Uniform finite-horizon transfer}

\begin{proposition}[Finite-horizon map shadowing]
\label{prop:shadowing}
Let \(T_k\) and \(\widehat T_k\) be paired update maps (for example, the
projected and ambient updates for the same realized batch).  Suppose both
trajectories remain in a set on which every \(T_k\) is \(L\)-Lipschitz and
\[
\sup_{k,z}\|\widehat T_k(z)-T_k(z)\|\le\varepsilon.
\]
If \(z_{k+1}=T_k(z_k)\),
\(\widehat z_{k+1}=\widehat T_k(\widehat z_k)\), and
\(\widehat z_0=z_0\), then
\[
\|\widehat z_k-z_k\|
\le
\varepsilon\sum_{j=0}^{k-1}L^j.
\]
Assume in addition that
\(|D_\mu(x)-D_\mu(y)|\le M\|x-y\|\) throughout the set.  Whenever
\[
|D_\mu(z_k)|
>
M\varepsilon\sum_{j=0}^{k-1}L^j,
\]
the two values \(D_\mu(z_k)\) and \(D_\mu(\widehat z_k)\) have identical
signs.
\end{proposition}

\begin{proof}
Let \(d_k=\|\widehat z_k-z_k\|\).  Then
\[
d_{k+1}
\le
\|\widehat T_k(\widehat z_k)-T_k(\widehat z_k)\|
+\|T_k(\widehat z_k)-T_k(z_k)\|
\le\varepsilon+Ld_k.
\]
Induction from \(d_0=0\) gives the trajectory bound.  The sign claim follows
from
\[
|D_\mu(\widehat z_k)-D_\mu(z_k)|
\le M\|\widehat z_k-z_k\|.
\]
\end{proof}

\begin{corollary}[Finite-token softmax shadowing]
\label{cor:softmax-shadowing}
Fix \(\eta_{\rm sm}>0\) and a compact normalized parameter set
\(\mathcal K\).  Couple the finite-token and population maps using the same
initial state, and suppose both trajectories remain in \(\mathcal K\) for
\(0\le k\le H\).  Then, for a fixed finite \(H\),
\[
\max_{k\le H}\|\widehat\zeta_k-\zeta_k\|
=
O_{\Pr}\left(
\frac{C_{\mathcal K,\eta_{\rm sm}}}{\sqrt N}
\sum_{j=0}^{H-1}L^j
\right),
\]
where \(L\) is a Lipschitz constant for the population maps on
\(\mathcal K\).  If \(S=\sup_{\mathcal K}|B|/\sqrt{\eta_{\rm sm}}\), the
pointwise constant has scale \(e^{S^2/2}\) up to polynomial factors, while
the uniform-envelope constant retains an exponential dependence on \(S^2\).
In particular, if the minimum population margin from \(D_\mu=0\) dominates
the displayed error scale, all finite-token and population basin labels up
to time \(H\) agree with probability tending to one.
\end{corollary}

\begin{proof}
For fixed \(S<\infty\), the classes
\[
\{e^{sz}z^q:|s|\le S,\ q\in\{0,1,2\}\}
\]
have square-integrable Gaussian envelopes and are finite-dimensional smooth
parametric classes.  Their empirical averages therefore converge uniformly
at rate \(N^{-1/2}\).  The population denominator is
\(\exp(s^2/2)\ge1\), so the same uniform rate holds for
\(\widehat m_N(s)-s\) and \(\widehat v_N(s)-1\), and hence for
\eqref{eq:finite-softmax-map} on \(\mathcal K\).  The pointwise covariance in
Proposition~\ref{prop:softmax-clt} exposes the factor \(e^{s^2}\).
Applying Proposition~\ref{prop:shadowing} proves the trajectory and sign
claims.
\end{proof}

The general sub-Gaussian theory of \citet{boursier2026softmax} gives
non-asymptotic output mean-square error of order
\(\log N/N^{c/\sigma^2}\) and parameter-gradient error of order
\(\log^2N/N^{c/\sigma^2}\), under their stated moment conditions.  Those
bounds can replace \(C_{\mathcal K,\eta_{\rm sm}}/\sqrt N\) in the same
shadowing argument.  The explicit Gaussian calculation above is sharper in
the rank-one sector and, unlike a risk-only transfer, resolves the
\(N^{-1/2}\) neighborhood in which the separatrix label remains random.

\end{document}